\documentclass[11pt]{article}
\usepackage{amsmath}
\usepackage{amssymb}
\usepackage{sidecap}

\usepackage{amsthm,wasysym}
\usepackage{algorithm, algorithmic}
\usepackage{upgreek}
\usepackage{sidecap}
\usepackage{fullpage}
\usepackage{graphicx} 
\RequirePackage[algo2e,ruled]{algorithm2e}

\newtheorem{theorem}{Theorem}
\newtheorem{lemma}[theorem]{Lemma}

\newtheorem{definition}[theorem]{Definition}
\newtheorem{proposition}{Proposition}
\newtheorem{remark}{Remark}

\newtheorem{claim}{Claim}

\newcommand{\cN}{\mathcal{N}}
\newcommand{\cM}{\mathcal{M}}

\newcommand{\into}{\rightarrow}

\newcommand{\cC}{\mathcal{C}}

\newcommand{\cO}{\mathcal{O}}
\newcommand{\cP}{\mathcal{P}}
\newcommand{\cI}{\mathcal{I}}
\newcommand{\cS}{\mathcal{S}}
\newcommand{\cL}{\mathcal{L}}

\newcommand{\N}{\mathbb{N}}

\newcommand{\cut}{\phi}

\newcommand{\er}{R}

\newcommand{\mb}{\mathbb{B}}
\newcommand{\smb}{\mathbb{B}^{\circ}}
\newcommand{\seqset}[1]{{{#1}^\text{*}}}
\newcommand{\cseq}{\mathcal{S}}
\newcommand{\cpp}{\mathbb{PP}}
\newcommand{\ZZ}{\mathbb{Z}_K}

\newcommand{\rwcs}{\upvarphi}

\newcommand{\seq}[1]{{#1}}
\newcommand{\id}[1]{{[#1]}}  % Iverson bracket notation 
\newcommand{\szm}{\operatorname{sim}}
\newcommand{\mct}{\Phi}
\newcommand{\cl}{y} % Class
\newcommand{\clv}{\by}
\newcommand{\siml}{y}
\newcommand{\trp}{\top}

\newcommand{\one}{\ensuremath{\mathbf{1}}}

\newcommand{\E}{\ensuremath{\mathbb{E}}}

\newenvironment{proofof}[1]{\par\noindent{\bf Proof of~\protect{#1}:}}{\hfill$\Box$\\}
\newcommand{\field}[1]{\mathbb{#1}}

\newcommand{\wh}{\widehat}
\newcommand{\tr}{\mbox{\sc tr}}

\newcommand{\by}{\boldsymbol{y}}

\newcommand{\bx}{\boldsymbol{x}}
\newcommand{\bw}{\boldsymbol{w}}
\newcommand{\bu}{\boldsymbol{u}}
\newcommand{\bq}{\boldsymbol{q}}
\newcommand{\bv}{\boldsymbol{v}}
\newcommand{\bp}{\boldsymbol{p}}
\newcommand{\be}{\boldsymbol{e}}
\newcommand{\bzero}{\boldsymbol{0}}

\newcommand{\R}{\field{R}}

\newcommand{\scM}{\mathcal{M}}

\newcommand{\vct}{\mbox{\sc vec}}

\newcommand{\cX}{\mathcal{X}}
\newcommand{\cY}{\mathcal{Y}}

\newcommand{\ignore}[1]{}

\DeclareMathOperator*{\argmax}{argmax} % This is the right one to use, no?
\newcommand{\trace}{\tr}
\newcommand{\MJH}[1]{\footnote{\color{red}MJH: #1}}
\newcommand{\WM}{lw-twma-94}
\newcommand{\shaimetric}{SSN04}
\newcommand{\trackingexpert}{HW98}

\title
{Online Similarity Prediction of Networked Data\\ from Known and Unknown Graphs}

\author{
Claudio Gentile \\ 
DiSTA, Universit\`a dell'Insubria, Italy\\
\texttt{claudio.gentile@uninsubria.it}
\and
Mark Herbster \\ 
Department of Computer Science, University College London \\
 Gower Street, London, WC1E 6BT, UK\\
\texttt{m.herbster@cs.ucl.ac.uk}
\and
Stephen Pasteris \\ 
Department of Computer Science, University College London \\
 Gower Street, London, WC1E 6BT, UK\\
\texttt{s.pasteris@cs.ucl.ac.uk}
}

\begin{document}

\maketitle
\vspace{-0.2in}
\begin{abstract}
We consider online similarity prediction problems over networked data.
%In this setting an algorithm sequentially receives  pairs of vertices from a graph and must predict, 
%for each pair, if they share the same label.   
We begin by relating this task to the more standard 
class prediction problem, showing that, given an arbitrary algorithm for class prediction,  we can 
construct an algorithm for similarity prediction with ``nearly'' the same mistake bound, and vice versa.
After noticing that this general construction is computationally infeasible,
we target our study to {\em feasible} similarity prediction algorithms on networked data. 
We initially assume that the network structure is {\em known} to the learner. 
Here we observe that Matrix Winnow~\cite{w07} has a near-optimal mistake guarantee, 
at the price of cubic prediction time per round.
This motivates our effort for an efficient implementation of a Perceptron 
algorithm 
with a weaker mistake guarantee but with only poly-logarithmic prediction time.  
Our focus then turns to the challenging case of networks whose structure is initially 
{\em unknown} to the learner. In this novel setting, where the network structure is only 
incrementally revealed, we obtain a mistake-bounded algorithm with a quadratic prediction 
time per round.
\end{abstract}

\vspace{-0.1in}
\section{Introduction}
\vspace{-0.05in}
\iffalse
As an example you are the manager of a network (graph).  There has been a catastrophic disaster that 
``breaks'' a number of edges in your network.  
You are incrementally tasked with predicting if two vertices remained connected (through an unbroken path) you receive  as feedback only ``connected'' or ``disconnected.'' 
As a member of security agency, you are tasked with identifying pairs of individuals who are in the same ``gang.''   
For your first identification task you receive a graph.
Consider the problem of online community identification.  Each trial you are given a pair of vertices from a graph then you predict if they are in class/cluster/community and finally you receive 
\fi
%
%
The study of networked data has spurred a large amount of research efforts. Applications like 
spam detection, product recommendation, link analysis, community detection, are by now well-known tasks
in Social Network analysis and E-Commerce. 
In all these tasks, networked data are typically viewed as graphs,
where vertices carry some kind of relevant information (e.g., user features in a social network),
and connecting edges reflect a form of semantic similarity between the data associated with the incident vertices. 
Such a similarity ranges from friendship among people in a social network to common user's reactions to
online ads in a recommender system, from functional relationships among proteins in a protein-protein
interaction network to connectivity patterns in a communication network.
Coarsely speaking, similarity prediction aims at inferring the existence of new pairwise relationships 
based on known ones. These pairwise constraints, which specify whether two objects belong to the same class 
or not,
% (sometimes called ``must-link" and ``cannot-link" constraints in clustering with side information
% (e.g.,~\cite{bdsy99,dbe99}) 
may arise directly from domain knowledge or be available with little human effort. 

%Because these similarity relationships are often transitive, one might
%cast the similarity prediction problem over networked data as a {\em metric learning} problem on a graph.  

There is a wide range of possible means of capturing the structure of a graph in this learning context: 
through combinatorial and classical graph-theoretical methods (e.g., \cite{gy03}); through spectral approachs 
(e.g. \cite{BMNLargeGraphs,Hain07}), using convex duality and resistive geometry (e.g., \cite{HL09}), 
and even algebraic methods~(e.g., \cite{Risi}).
In many of these approaches, the underlying assumption is that the graph structure
is largely known in advance (a kind of ``transductive" learning setting), 
and serves as a way to bias the inference process, so as to implement the principle that 
``connected vertices tend to be similar."
Yet, this setting is oftentimes unrealistic and/or infeasible. 
For instance, a large online social network with millions of vertices and tens of millions of edges 
hardly lends itself to be processed as a whole via a Laplacian-regularized optimization approach or, 
even if it does (thanks to the computationally powerful tools currently available), 
it need not be known ahead of time. As a striking example, if we are representing a security 
agency, and at each point in time we receive a ``trace'' of communicating individuals, we still might
want to predict whether a given pair in the trace belong to the same ``gang''/community, even if the
actual network of relationships is unknown to us. So, in this case, we are incrementally
learning similarity patterns among individuals while, at the same time, exploring the network.
Another important scenario of an unknown network structure is when the network itself
{\em grows} with time, hence the prediction algorithms are expected to somehow adapt to its 
temporal evolution.

% ***********
% As another example, in a protein-protein interaction network, one typically collects over time 
% positive interactions between proteins from experimental assays,  
% while negative interactions are often sampled randomly among non-interacting pairs. 
% Hence, more often than not, we can hardly tell whether the unobserved edges 
% represent actual inhibitory interactions or are really absent from the network.
% ***********

\noindent{\bf Our results. }
We study online similarity prediction over graphs in two models.   
One in which 
the graph is {\em known}\footnote{The reader should keep in mind that 
while the data at hand may not be natively graphical,
it might still be convenient in practice to artificially generate a graph for 
similarity prediction, since the graph may encode side information that is 
otherwise unexploitable.} a priori to the learner,
and one in which it is {\em unknown}.
In both settings there is an {\em undisclosed} labeling of a graph so that each vertex is the member of one of $K$ classes.  Two vertices are similar if they are in the same class and dissimilar otherwise.  The learner receives an online sequence of vertex pairs and similarity feedback.  On the receipt of a pair the learner then predicts if the pair is similar.  The true pair label, {\sc similar} or {\sc dissimilar}, is then received and the goal of the learner is to minimize mistaken predictions.  Our aim in both settings is then to bound the number of prediction mistakes over
an arbitrary (and adversarially generated) sequence of pairs.

In the model where the graph is known, we first show via reductions to online vertex classification methods on graphs 
(e.g., \cite{Her08,HL09,HP07,HPW05,HLP09,HPR09, CGV09b, CGVZ10, CGVZ10b}, and references therein),
 that a suitable adaptation of the Matrix Winnow algorithm~\cite{w07} 
readily provides an almost optimal mistake bound. This adaptation amounts to 
sparsifying the underlying graph $G$ via a random spanning tree, whose diameter is then
shortened by a known rebalancing technique~\cite{HLP09,CGVZ10}.
%Unfortunately, with a cubic prediction time per round, the resulting algorithm seems impractical for large networks.
Unfortunately, due to its computational burden (cubic time per round), the resulting algorithm
does not provide a satisfactory answer to actual deployment on large networks.
%Thus motivated, we show that using the same rebalanced sparsification with a matrix perceptron can be implemen
%Therefore, we resort to showing that an analogous adaptation of a Matrix Perceptron
%algorithm does indeed deliver a much more attractive answer 
%(thanks to its  poly-logarithmic time per round), though with an inferior online 
%prediction performance guarantee.
Therefore, we develop an analogous adaptation of a Matrix Perceptron
algorithm that delivers a much more attractive answer 
(thanks to its  poly-logarithmic time per round), though with an inferior online 
prediction performance guarantee.

The unknown model is identical to the known one, 
except that the learner does not initially receive the underlying graph $G$. 
Rather, $G$ is incrementally revealed, as now when the learner receives a pair it also receives as side information an adversarially generated path within $G$ connecting the vertices of the pair.
Here, we observe that the machinery we used for the known graph case is inapplicable.
%neither  Matrix Winnow, nor the Matrix Perceptron algorithms are applicable.  
Instead, we design and analyze an algorithm which may be interpreted as 
a matrix version of an adaptive $p$-norm Perceptron~\cite{gls97,g03} with the 
relatively efficient quadratic running time per round.

% !TEX root = main7.tex

\noindent{\bf Related work. }
%
%{\bf CG: this is a paper on similarity which is similar to other papers}
%
This paper lies at the intersection between online learning on graphs 
and matrix/metric learning. Both fields include a substantial amount of work, 
so we can hardly do it justice here. Below we outline some of the main 
contributions in matrix/metric learning, with a special emphasis on those we believe 
are most related to this paper. Relevant papers in online class prediction 
on graphs will be recalled in Section \ref{sec:spg}.

Similarity prediction on graphs can be seen as a special case of matrix learning.
Relevant works on this subject include~\cite{matrixeg,w07,ccg10,kst12} --
see also~\cite{HKSS12} for recent usage in the context of online cut prediction. 
In all these papers, special care is put into designing appropriate regularization
terms driving the online optimization problem, the focus typically being on
spectral sparseness. When operating on graph structures with Laplacian-based
regularization, these algorithms achieve mistake bounds depending on functions 
of the cut-size of the labeled graph -- see Section \ref{s:efficient}. 
Yet, in the absence of further efforts, their scaling properties make them 
inappropriate to practical usage in large networks.
{\em Metric} learning is also relevant to this paper. Metric learning is a special 
case of matrix learning where the matrix is positive semi-definite. Relevant
references include~\cite{\shaimetric,D+07,m08,zy07,cgy12}. Some of these papers
also contain generalization bound arguments. Yet, no specific concerns are cast
on networked data frameworks. Related to our bidirectional reduction from 
class prediction to similarity prediction is the thread of papers on kernels
on pairs (e.g., \cite{BH04,m08,LLT08,fflt12}), where kernels over pairs of 
objects are constructed as a way to measure the ``distance" between the two referenced pairs. %at argument.
The idea is then to combine with any standard kernel algorithm.  
The so-called matrix completion task (specifically, the recent reference~\cite{kr12}) is also 
related to our work. In that paper, the authors introduce
a matrix recovery method working in noisy environments, which incorporates
both a low-rank and a Laplacian-regularization term. 
The problem of recovery of low-rank matrices has extensively been studied
in the recent statistical literature (e.g., \cite{cr09,ct10,gr11,rt11,nw10,klt11}, 
and references therein), the main concern being bounding the recovery error rate, but 
disregarding the computational aspects of the selected estimators. 
Moreover, the way they typically measure error rate is not easily 
comparable to online mistake bounds.
Finally, the literature on semisupervised clustering/clustering with side information
(\cite{bdsy99,dbe99} -- see also \cite{rh12} for a recent reference on spectral approaches
to clustering) is related to this paper, since the similarity feedback can 
be interpreted as a must-link/cannot-link feedback. Nonetheless, their formal statements
%, if any, 
are fairly different from ours.

To summarize, whereas we are motivationally close to~\cite{kr12},
from a technical viewpoint, we are perhaps closer to \cite{SSN04,matrixeg,w07,ccg10,HKSS12,kst12},
as well as to the literature on online learning on graphs.

% \noindent{\bf Organization of the paper. }
%
Before delving into the graph-based similarity problem, we start off 
%in the next section 
by investigating the problem of similarity prediction in abstract terms,
showing that similarity prediction reduces to classification, and vice versa. %% vice versa (two words)
This will pave
the way for all later results.

% !TEX root = main7.tex 

\vspace{-0.1in}
\section{Online class and similarity prediction}\label{sec:corr}
\vspace{-0.05in}
In this section we examine the correspondence in predictive performance (mistake bounds) between the classification and similarity prediction frameworks.  
%
%We prove in Theorem~\ref{thm:doubleequiv} that a mistake bound for a $K$-class 
%classification problem implies that there exists a (potentially inefficient) algorithm with a 
% mistake bound which increases by a factor of no more than ``$5\log K$'' for an equivalent 
%similarity problem.  Conversely we also show that  a mistake bound for similarity implies 
% a mistake bound for classification with only an additive term of ``$K$.''   

%This will be exploited to give lower bounds for similiarity prediction over graphs in Section~\ref{sec:spg}.   

%Theorem~\ref{thm:lognec} shows the tightness of the induced ``$\log K$ factor.
%\subsection*{Preliminaries}
\noindent{\bf Preliminaries. }
% We write $S' \subseteq S$ to denote that $S'$ is a subset of $S$ as well to denote that $S'$ is a {\em subsequence} of $S$. 
% if $S'$ is a sequence that can be obtained by ``deleting'' elements of the sequence $S$.  
The  set of all finite sequences from a set $\cX$ is denoted $\seqset{\cX}$.
We use the Iverson bracket notation $\id{\mbox{\sc predicate}}=1$ if the predicate is true and $\id{\mbox{\sc predicate}}=0$ if false.
In $K$-{\em class prediction} in the online mistake bound model, an {\em example sequence} $\seq{(x_1,y_1),\ldots,(x_{T},y_{T})}\in  \seqset{(\cX\times \cY)}$ is revealed incrementally, where $\cX$ is a set of {\em patterns} 
and $\cY:= \{1,\ldots,K\}$ is the set of $K$ class\ {\em labels}.  The goal on the $t$-th trial is to predict  the class~$y_t$ given the previous  $t-1$  pattern/label pairs and $x_t$.  The overall aim of an algorithm is to minimize the number of its mistaken predictions.    In {\em similarity prediction}, examples are pairs of patterns with  ``similarity'' 
labels i.e., $((x',x''),y) \in \cX^2\times \cY_s$ with $\cY_s = \{0,1\}$.  We interpret $y\in \cY_s$ as {\sc similar} 
if $y=0$ and {\sc dissimilar} if $y=1$; we also introduce the convenient function $\szm(y',y'') := 1 - \id{y'=y''}$ 
which maps a pair of class labels $y', y'' \in \cY$
to a similarity label.  A {\em concept} is a function $f: \cX \into \cY$ 
that maps patterns to labels.  An   example sequence $S$ is {\em consistent} with a concept $f$ for classification if $(x,y) \in S \text{ implies }y=f(x)$ and for similarity if $((x',x''),y) \in S \text{ implies }y=  \szm(f(x'),f(x''))$.
%
%Observe then that $y^s = \overline{\cI[ y' =y'']}$ which is counterintuitively logically but corresponds to the intuition later exploited that to vectors are similar if the angle between them is zero.
\iffalse
In similarity prediction we receive a pattern pair $(x',x'')\in \cX^s = \cX\times\cX$  and a label from $\cY^s := \{{\sc similar},{\sc disimilar}\}$.
\fi
%
We use $M_A(S)$ to denote the number of prediction mistakes of the {\em online algorithm} $A$ on example sequence $S$.
%%% Informal but shorter definition above 
\iffalse
A prediction algorithm is a mapping $A: \seqset{(\cX\times \cY)} \into \cY^\cX$ from example sequences to prediction functions.  
Thus if $A$ is a prediction algorithm and $S=\seq{(x_1,y_1),\ldots,(x_{T},y_{T})}\in  \seqset{(\cX\times \cY)}$ is an example sequence then  the {\em online} prediction mistakes are 
\[ % better inline if we can prevent line overflow
M_A(S) := \sum_{t=1}^{T} \id{A(\seq{(x_1,y_1),\ldots,(x_{t-1},y_{t-1})})(x_t) \ne y_t}\,.
\]
\fi
Given an algorithm $A$, we define the {\em mistake bound} with respect to a concept $f$ as $\mb_A(f) := \max_S M_A(S)$,
the maximum being over all sequences $S$ consistent with $f$.
\begin{theorem}\label{thm:doubleequiv}
Given an  online {\em classification} algorithm $A_c$ one may construct a {\em similarity} algorithm $A_s$ 
such that if $S$ is any {\em similarity} sequence consistent with any concept $f$ then
\vspace{-0.05in}
\begin{equation}\label{eq:dmup}
M_{A_s}(S)  \le  5\,\mb_{A_c}(f) \log_2 K\,,
\end{equation}
%\vspace{-0.05in}
and given an  online {\em similarity} algorithm $A_s$ one may construct a {\em classification} algorithm $A_c$ 
such that if $S$ is any {\em classification} sequence consistent with any concept $f$ then
\vspace{-0.05in}
\begin{equation}\label{eq:dsimtoclass}
M_{A_c}(S)  \le  \mb_{A_s}(f) + K \,.
\end{equation}
\vspace{-0.05in}
\end{theorem}
\vspace{-0.2in}
The direct implementation of the similarity algorithm $A_s$ from the classification algorithm $A_c$ is infeasible, 
as its running time is exponential in the mistake bound.   
%Thus this motivates our search for feasible algorithm for similarity prediction.
%On the other hand, the implementation of classification algorithm $A_c$ out of the similarity algorithm $A_s$  is essentially as efficient as $A_s$.  
In Appendix~\ref{sec:proofofreductions} we prove a more general result (see Lemma~\ref{lem:equiv}) than 
in equation~\eqref{eq:dmup} which applies also to noisy sequences and to ``order-dependent'' 
bounds, as in the shifting-expert bounds in~\cite{\trackingexpert}.
We also argue (Appendix~\ref{sec:lognec}) that the ``$\log K$'' term in~\eqref{eq:dmup} is necessary.  
Observe that equation~\eqref{eq:dsimtoclass} implies a lower bound for similarity prediction if we have a lower bound for the corresponding class prediction problem with only a weakening by an additive ``$-K$'' term. 
% In Section~\ref{sec:spg} we apply
% (\ref{eq:dsimtoclass}) to get lower bounds for similarity prediction over graphs.

% !TEX root = main7.tex

%\section{A prospectus for similarity prediction over graphs}\label{sec:spg}
\vspace{-0.1in}
\section{Class and similarity prediction on graphs}\label{sec:spg}
\vspace{-0.05in}
%
%We now focus our attention onto the similarity prediction problem over graphs. 
We now introduce notation specific to the graph setting.
%
% We introduce necessary notation pertaining to graphs in section~\ref{sec:graphprelim}
%
% As similarity and class prediction are intimately related we  review results for class prediction 
% over graphs insection~
% \ref{sec:classgraph}.  We then summarize in section~\ref{sec:summary} the results that are obtainable 
% for (inefficient) 
% similarity prediction via version space arguments as well as through the reductions from 
% sections~\ref{sec:corr} and~\ref
% {sec:hs}.  Then in section~\ref{sec:echard} we argue that although efficient version space 
% predictions may be computed 
% in linear time with class feedback over trees~\cite{largediam}, that even on a path graph 
% computation of the natural 
% partition function is $\#P$-complete.
%
% \subsection{Preliminaries} \label{sec:graphprelim}
% Let us now consider the case when the similarity problem is described by 
%
Let then $G = (V,E)$ be an undirected and connected graph with $n = |V|$ vertices, $V = \{1,\ldots, n\}$, and 
$m = |E|$ edges.  The assignment of $K$ class labels to the vertices of a 
graph is denoted by a vector $\clv = (y_1, \ldots, y_n)$, where $y_i \in \{1,\ldots,K\}$
denotes the label of the $i$-th vertex among the $K$ possible labels. 
%Associated with each vertex $i$ is a class label $\cl_i \in \N_K$. 
The vertex-labeled graph will often be denoted by the pairing $(G,\clv)$.
Associated with each pair 
$(i,j) \in V^2$ of (not necessarily adjacent) vertices is a similarity label $\siml_{i,j} \in \{0,1\}$, 
where $\siml_{i,j} = 1$ if and only if $\cl_i \neq \cl_j$. As is typical of graph-based 
prediction problems (e.g., \cite{Her08,HL09,HP07,HPW05,HLP09,HPR09, CGV09b, CGVZ10, CGVZ10b},
and references therein), the graph structure plays the role of an inductive bias, where adjacent vertices tend to belong
to the same class. 
%There has been extensive study of the (two) class prediction problem on graphs~\cite{commentoutMJH}.  
The set of {\em cut-edges} in $(G,\clv)$ is denoted as
$\mct^{G}(\by) := \{(i,j) \in E\,:\, y_{i,j} = 1 \}$ (when nonambiguous, we abbreviate it to $\mct^G$),
and the associated cut-size as $|\mct^{G}(\by)|$.
The set of cut-edges with respect to class label $k$
%\in\N_K$ 
is denoted as $\mct_k^{G}(\by) := \{(i,j) \in E\,:\, k \in\{\cl_i,\cl_j\},\, y_{i,j} = 1\}$
(when nonambiguous, we abbreviate it to $\mct^G_k$).
Notice that $\sum_{k=1}^K |\mct_k^{G}(\by)| = 2 |\mct^{G}(\by)|$.
%
\iffalse
A central aim of predicting the labeling of a graph has been to predict with few mistakes when the ``complexity'' of the 
function labeling the graph is small.  A natural notion of the complexity of a labeling is the notion of the {\em cut 
(size)} $\mct^{G}(\bu) := \sum_{(i,j) \in E(G)} \cI[u_i \ne u_j]$ between classes (when nonambiguous we abbreviate to $
\mct$).  Thus it is simply the number of edges between vertices with distinct class labels.
\fi
%
We let $\Psi$ be the $m\times n$ (oriented and transposed) incidence matrix of $G$. 
Specifically, if we let the edges in $E$ be enumerated as $(i_1,j_1), \ldots, (i_{m},j_{m})$,
and fix arbitrarily an orientation for them (e.g., from the left endpoint to the right endpoint),
then $\Psi$ is the matrix that maps any vector $\bv  = (v_1, \ldots, v_n)^\top \in\R^{n}$ to the vector 
$\Psi\bv\in\mathbb{R}^{m}$, where $[\Psi\bv]_{\ell} = v_{i_{\ell}}-v_{j_{\ell}}$, $\ell = 1, \ldots, m$.
Moreover, since $G$ is connected, the null space of $\Psi$ is spanned by the constant vector 
$\one = (1, \ldots, 1)^\top$, that is, $\Psi\bv = \bzero$ implies that $\bv = c\one$, for some constant $c$.
We denote by $\Psi^+$ the ($n\times m$-dimensional) pseudoinverse of $\Psi$.   
The graph Laplacian matrix may be defined as  $L := \Psi^\trp\Psi$, thus
notice that $L^+ = \Psi^+(\Psi^+)^\top$. 
%
%CG: better use \bv than \by, as the latter is the vector of labels, 
%
% The quadratic form associated with the Laplacian may be viewed as a real-valued relaxation of the cutsize
% $|\mct_k^{G}(\by)|$; this is because 
% $\bv^{\trp} L \bv   = \bv^\trp \Psi^\trp\Psi \bv = \sum_{(i,j)\in E} (v_i-v_j)^2$.   
If $G$ is identified with a resistive network such that each edge is a unit resistor,
then the {\em effective resistance} $\er^G_{i,j}$ between a pair of vertices $(i, j)\in V^2$ can be defined as
$\er^G_{i,j} = (\be_i - \be_j)^\trp L^+(\be_i-\be_j)$,
where  $\be_i$ is the $i$-th vector in the canonical basis of $\R^n$.
%
%$\er_{i,j} = (\argmin_{u\in\R^n} u^\trp L u : u_i-u_j=1)^{-1}$.  
%
When $(i,j) \in E$ then $\er^G_{i,j}$ also equals the probability that a spanning tree of $G$ 
drawn uniformly at random (from the set of all spanning trees of $G$) includes $(i,j)$ as one of its $n-1$ edges (e.g., \cite{lp10}). The resistance diameter of $G$ is $\max_{(i,j)\in V^2} \er^G_{i,j}$. 
It is known that the effective resistance defines a metric over the vertices of $G$. Moreover, 
when $G$ is actually a tree, then $\er^G_{i,j}$ corresponds to the number of edges in the (unique)
path from $i$ to $j$. Hence, in this case, the resistance diameter of $G$ coincides with its (geodesic) 
diameter.

\vspace{-0.1in}
\subsection{Class prediction on graphs}\label{sec:classgraph}
\vspace{-0.05in}
Roughly speaking, algorithms and bounds for sequential class prediction on graphs split 
between two types: Those which approximate the original graph with a tree or 
those that maintain the original graph. 
By approximating the graph with a tree, extremely efficient algorithms are obtained 
with strong optimality guarantees. 
By exploiting the full graph, algorithms are obtained which take advantage of the 
connectivity to achieve sharp bounds when the graph contains, e.g., dense clusters. 
Relevant literature on this subject includes \cite{HPW05,HP07,Her08,HPR09,HLP09,CGV09b,HL09,CGVZ10}.
Known representatives of the first kind are upper bounds of the form
$\cO(|\mct^T|(1+ \log \frac{n}{|\mct^T|}))$ \cite{HLP09} or of the form $\cO(|\mct^T| \log D_T)$ \cite{CGV09b},
where $T$ is some spanning tree of $G$, and $D_T$ is the (geodesic) diameter of $T$. 
In particular, if $T$ is drawn uniformly at random, the above turn to bounds on the {\em expected}
number of mistakes of the form $\cO(\E[|\mct^T|] \log n)$, where 
$\E[|\mct^T|]$ is the resistance-weighted cut-size of $G$, 
$\E[|\mct^T|] = \sum_{(i,j)\in \Phi^G(\by)} \er^G_{i,j}$~,
which can be far smaller than $|\Phi^G|$ when $G$ is well connected.
Representatives of the second kind are bounds of the form 
$\cO(\rho + |\mct^G|\,R_{\rho})$ \cite{Her08,HL09}, where $\rho$ is the number of balls in a cover of 
the vertices of $G$ such that $R_{\rho}$ is the maximum over 
the resistance diameters of the balls in the cover.   
Since resistance diameter lower bounds geodesic diameter,
this alternative approach leverages a different connectivity structure of the graph
than the resistance-weighted cut-size.

In all of the above mentioned works, 
% discussed with a single exception\footnote{An exception is that 
% in~\cite{boonserm} they showed that the specific algorithm of~\cite{HLP09} could be 
% lifted to $K$-class case.} 
the bounds and algorithms are for the $K = 2$ class prediction case. 
In Appendix \ref{ssa:missingprospectus} 
we argue for a simple reduction that will raise a variety of 
cut-based algorithms and bounds from the two-class to the $K$-class case. 
Specifically, a two-class mistake bound of the form $M \le c |\mct^{G}(\by)|\,\,\,\,\forall \by \in \{0,1\}^n$,
for some $c \geq 0$ easily turns into a $K$-class mistake bound of the form 
$M \le 2c |\mct^{G}(\by)|\,\,\,\,\forall \by \in \{1,\ldots, K\}^n$,
where $K$ need not be known in advance to the algorithm.
Therefore, bounds of the form $\cO(\E[|\mct^T|] \log n)$ also hold in the multiclass setting.

On the lower bound side, \cite{CGV09b} contains an argument showing (for the $K = 2$ class case)
that for any $\phi \geq 0$ a labeling $\by$ exists such that any algorithm will make at least 
$\phi/2$ mistakes while $\E[|\mct^T|] < \phi$.
In short, $\Omega(\E[|\mct^T|])$ is also a lower bound on the number of mistakes
in the class prediction problem on graphs.  
When combined with Theorem \ref{eq:dmup} in Section \ref{sec:corr},
the above results immediately yield upper and lower bounds for the similarity prediction problem
over graphs.
\begin{proposition}\label{p:boundsimilarity}
Let $(G,\by)$ be a labeled graph, and $T$ be a random spanning tree of $G$.
Then an algorithm exists for the similarity prediction problem on $G$ 
whose expected number of mistakes $\E[M]$ satisfies\ 
\(
\E[M] = \cO(\E[|\mct^T(\by)|]\,\log K\,\log n)~.
\)
Moreover, for any $\phi \geq 0$ a $K$-class labeling $\by$ exists such that 
any similarity prediction algorithm on $G$ will make at least $\phi/2-K$ 
mistakes while $\E[|\mct^T|] < \phi$.
\end{proposition}
The upper bound above refers to a computationally inefficient algorithm and,
clearly enough, more direct version space arguments would lead to similar results.
Section~\ref{s:efficient} contains a more efficient approach to similarity prediction
over graphs. 

To close this section, we observe that the upper bounds 
on class predictions of the form $\cO(\E[|\mct^T|] \log n)$
taken from \cite{HLP09,CGVZ10} are essentially relying on linearizing the 
graph $G$ into a path graph, and then predicting optimally on it via an efficient
Bayes classifier (aka Halving Algorithm, e.g., \cite{litt88}). One might wonder
whether a similar approach would directly apply to the similarity prediction problem.  
We now show that exact computation of the probabilities of the Bayes classifier for a path 
graph is \#P-complete under similarity feedback.

The {\sc Ising} distribution over graph labelings ($\by\in\{1,2\}^n$) is defined as $p(\by) \propto 2^{-\beta|\cut^G(\by)|}$.   
Given a set of vertices and associated labels, 
the marginal distribution at each vertex can be computed in linear time when the graph is a path (\cite{pearlOrigBp82}). 
In~\cite{HLP09} this simple fact was exploited to give an efficient class prediction algorithm by a particular linearization of a graph to a path graph.   
The equivalent problem in similarity prediction requires us to compute marginals given a set of pairwise constraints.  The following theorem shows that computing the partition function (and hence the relevant marginals) of the {\sc Ising} distribution on a path with pairwise label constraints is  \#P-complete.
\begin{theorem}\label{t:hardness}
Computing the partition function of the (ferromagnetic) {\sc Ising} model on a path 
graph with pairwise constraints is  \#P-complete,
where an {\sc Instance} is an $n$-vertex path graph $P$, a set of pairs $\cC \subset \{1,\ldots,n\}^2$, 
and a natural number, $\beta$, presented in unary, %notation, 
and the desired {\sc Output} 
is the value of the partition function, \ \
%\begin{equation}
\(
Z_P(\cC,\beta) :=  \sum_{\by\in\{1,2\}^n : \{(y_i = y_j)\}_{(i,j)\in\cC}}2^{-\beta|\cut^P(\by)|}\, .
\)
%\end{equation}
\end{theorem}
%
\iffalse
%
\begin{theorem}\label{t:hardness}
Computing the partition function of the (ferromagnetic) {\sc Ising} model on a path 
graph with pairwise constraints is  \#P-complete,
where an {\sc Instance} consists of an $n$-vertex path graph $P$, a set of pairs $\cC \subset \{1,\ldots,n\}^2$, 
and a natural number, $\beta$, presented in unary notation, 
and the desired {\sc Output} 
is the value of the partition function, \ \
%\begin{equation}
\(
Z_P(\cC,\beta) :=  \sum_{\by\in\{1,2\}^n : \{(y_i = y_j)\}_{(i,j)\in\cC}}2^{-\beta|\cut^P(\by)|}\, .
\)
%\end{equation}
\end{theorem}
%
\fi
Thus computing the {\em exact} marginal probabilities on even a path 
graph will be infeasible (given the hardness of  \#P).   
As an alternative, in the following section we discuss the application of the 
Matrix Perceptron and Matrix Winnow algorithms to similarity prediction.

% !TEX root = main7.tex
%%%% Section DEFINES %%%%%
\newcommand{\htheta}{\wh{\theta}}
%%%% END Section DEFINES %%%%%

\vspace{-0.1in}
\subsection{Similarity prediction on graphs}\label{sec:hs}
\vspace{-0.05in}
In Algorithm~\ref{a:2} we give a simple application of the Matrix Winnow (superscript ``w")
and Perceptron (superscript ``p") algorithms to similarity prediction on graphs. The key aspect of the construction 
(common to many methods in metric learning) is the creation of rank one 
matrices which correspond to similarity ``instances'' (see~\eqref{eq:siminstances}).
%% ----------------------------------------------------------------------------
%\hspace{-0.4in}
\begin{algorithm2e}[t]
\SetKwSty{textrm}
\SetKwFor{For}{{\bf For}}{}{}
\SetKwIF{If}{ElseIf}{Else}{if}{}{else if}{else}{}
\SetKwFor{While}{while}{}{}
%\vspace{0.1cm}
{\bf Input:}~ Graph $G = (V,E)$, $|V|= n$, with Laplacian $L = \Psi^\top\Psi$, and
$R := \max_{(i,j)\in V^2} R^G_{i,j}$;\\
{\bf Parameters:} 
Perceptron threshold $\htheta^{\mbox{\scriptsize p}} = R^2 $; \ \  
Winnow threshold $\htheta^{\mbox{\scriptsize w}} = \frac{\eta}{e^{\eta}-e^{-\eta}}\,\frac{1}{R\,|\Phi^G|} $, 
Winnow learning rate $\eta = 1.28$; \\ 
{\bf Initialization:} $W^{\mbox{\scriptsize p}}_{0} = \bzero\in \R^{m\times m};
\quad W^{\mbox{\scriptsize w}}_{0} = \frac{1}{m}\,I \in \R^{m\times m}$;\\
\For{ $t=1,2,\ldots, T$ :}
{
\vspace{-0.1in}
{
\begin{itemize}
\item Get pair of vertices $(i_t,j_t) \in V^2$, and construct similarity instances,
\vspace{-0.05in}
\begin{equation}\label{eq:siminstances}
X^{\mbox{\scriptsize p}}_t = (\Psi^+)^\top(\be_{i_{t}}-\be_{j_{t}})(\be_{i_t}-\be_{j_t})^\top\Psi^+;\ \
X^{\mbox{\scriptsize w}}_t = 
\frac{(\Psi^+)^\top(\be_{i_{t}}-\be_{j_{t}})(\be_{i_t}-\be_{j_t})^\top \Psi^+}{(\be_{i_t}-\be_{j_t})^\top L^{+}(\be_{i_t}-\be_{j_t})}\,;
\end{equation}
\vspace{-0.25in}
\item Predict: $\hat{y}^{\mbox{\scriptsize p}}_t = \id{\trace((W^{\mbox{\scriptsize p}}_{t-1})^{\top} X^{\mbox{\scriptsize p}}_t)>\htheta^{\mbox{\scriptsize p}}}$;\quad\quad $\hat{y}^{\mbox{\scriptsize w}}_t = \id{\trace((W^{\mbox{\scriptsize w}}_{t-1})^{\top} X^{\mbox{\scriptsize w}}_t)>\htheta^{\mbox{\scriptsize w}}}$;
\vspace{-0.05in}
\item Observe $y_t \in \{0,1\}$ and, if mistake ($y_t \neq {\hat y_t}$), update
\[ 
W^{\mbox{\scriptsize{p}}}_t \leftarrow W^{\mbox{\scriptsize{p}}}_{t-1}
+ 
(y_t - \hat{y}^{\mbox{\scriptsize p}}_t)\,X^{\mbox{\scriptsize p}}_t;
\quad\quad 
\log W^{\mbox{\scriptsize w}}_t \leftarrow \log W^{\mbox{\scriptsize w}}_{t-1}
+ 
\eta\,(y_t - \hat{y}^{\mbox{\scriptsize w}}_t)\,X^{\mbox{\scriptsize w}}_t~. 
\]
\end{itemize}
}
}
\vspace{-0.33in}
\caption{Perceptron and Matrix Winnow algorithms on a graph\label{a:2}}
\vspace{-0.03in}
\end{algorithm2e}
%% ----------------------------------------------------------------------------
%
We then use the standard analysis of the Perceptron~\cite{Novikoff62} and Matrix Winnow~\cite{w07} algorithms 
with appropriate thresholds to obtain Proposition~\ref{prop:percwin}.  
A key observation is that the squared Frobenius norm of the (un-normalized) instance matrices is bounded by the squared resistance diameter of the graph, and the squared Frobenius norm of the (un-normalized) ``comparator'' 
matrix is bounded by the cut-size squared $|\Phi^G|^2$.
% $\sum_{k=1}^K |\Phi_k^G|^2 \le|\Phi^G|^2$.
%
\vspace{-0.1in}
\begin{proposition}\label{prop:percwin}
Let $(G,\by)$ be a labeled graph and let $\Psi$ be the (transposed) incidence matrix associated
with the Laplacian of\ $G$. Then, if we run the Matrix Winnow and Perceptron algorithms 
%(see Algorithm Figure~\ref{a:2}) 
with similarity instances constructed from $\Psi$, we have the following
mistake~bounds:
\vspace{-0.05in}
% for Matrix Winnow and Matrix Perceptron, respectively:
\[
M^{\mbox{\sc w}} = \cO\Bigl(|\Phi^G| \max_{(i,j)\in V^2} \er_{i,j}^G \log n \Bigl)\ \ \text{ and }\ \  
M^{\mbox{\sc p}} = \cO\Bigl(|\Phi^G|^2 \max_{(i,j)\in V^2} (\er_{i,j}^G)^2 \Bigl)~.
\]
\end{proposition}
A severe drawback of both these algorithms is that on a generic graph, initilization 
requires computing a pseudo-inverse (typically cubic time), and furthermore the update of Matrix 
Winnow requires a cubic-time computation of an eigendecomposition (to compute matrix exponentials) 
on each mistaken trial.\footnote
{
Additionally, there is a tuning issue related to Matrix Winnow, since the 
threshold $\htheta^{\mbox{\scriptsize w}}$
depends on the (unknown) cut-size.
}  
In the following section, we focus on a construction based on a graph approximation % from~\citep{HLP09} 
for 
which we develop an efficient implementation of the Perceptron algorithm 
which will require only poly-logarithmic time per round.
% trial to predict.

% !TEX root = main7.tex

%\vspace{-0.1in}
\section{Efficient similarity prediction on graphs}\label{s:efficient}
%
%\vspace{-0.05in}
Relying on the notation of Section~\ref{sec:spg},
we turn to efficient similarity prediction on graphs.
We present adaptations of Matrix Winnow and Matrix Perceptron to the case when the original
graph $G$ is sparsified through a linearized and rebalanced random spanning tree of $G$. 
This sparsification technique, called Binary Support Tree (BST) in~\cite{HLP09},
brings the twofold advantage of yielding improved mistake bounds and
faster prediction algorithms. More specifically, the use of a BST replaces 
the (perhaps very large) resistance diameter term $\max_{(i,j)\in V^2} R^G_{i,j}$ in the mistake bounds of
Proposition \ref{prop:percwin} 
by a logarithmic term, the other term in the mistake bound becoming (when dealing with the expected 
number of mistakes) only a logarithmic factor larger than the (often far smaller) sum of 
the resistance-weighted cut-sizes in a spanning tree.
Moreover, when combined with the Perceptron algorithm, a BST allows us to develop a 
very fast implementation whose running time per round is poly-logarithmic in $n$, rather than
cubic, as in Matrix Winnow-like algorithms.

Recall that a uniformly random spanning tree of an unweighted graph
can be sampled % with a random walk 
in expected time $\mathcal{O}(n\ln n)$ for ``most'' graphs~\cite{Bro89}.
Using the nice algorithm of \cite{Wil96}, the expected time reduces 
to $\mathcal{O}(n)$ ---see also the work of~\cite{AAKKLT08}. 
However, all known techniques take expected time $\Theta(n^3)$ in certain pathological cases.
%
% The BST is a technique we borrow from~\cite{HLP09} (see also \cite{CGVZ10}). 

In a nutshell, a BST $B$ of $G$ is
a full balanced binary tree whose leaves correspond to the vertices in\footnote
{
We assume w.l.o.g. that $n = |V|$ is a power of 2. Otherwise, we may add dummy ``leaves".
}  
$V$. In order to construct $B$ from $G$, we first extract a random spanning tree $T$ of $G$,
then we visit $T$ through a depth-first visit, order its vertices according to this visit
eliminating duplicates (thereby obtaining a path graph $P$), and finally we build $B$ on top 
of $P$. Since $B$ has $2n-1$ vertices, we extend the class labels from leaves to internal vertices 
by letting, for each internal vertex $i$ of $B$, $y_i$ be equal to the class label of $i$'s 
left child. Figure~\ref{f:bst} illustrates the process.
\begin{figure}[t!]
\centering
\vspace{-0.35in}
\includegraphics[width=0.75\textwidth]{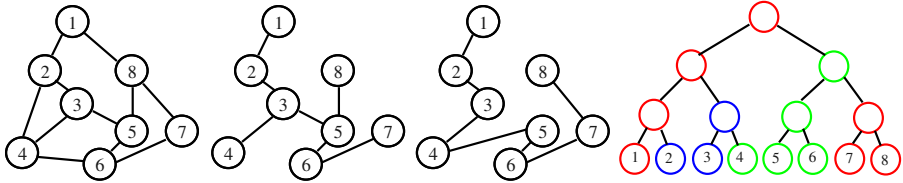}
\vspace{-4.9in}
\caption{{\small {\bf From left to right:} The graph $G$; a random spanning tree $T$ of $G$ 
(note that the vertices are numbered by the order of a depth-first visit of $T$, starting from root vertex 1;
the path graph $P$ which follows the order of the depth-first visit of $T$;
the BST built on top of $P$. Notice how the class labels of the 8 vertices in $V$ (corresponding to the
three colors) are propagated upwards.}
%\vspace{-0.3in}
}
\label{f:bst}
\end{figure}
A simple adaptation of~\cite{HLP09} (Section 6 therein) shows that for any class $k = 1, \ldots, K$ we have
$|\Phi^B_k|\leq 2\,|\Phi^T_k|\,\log_2 n$.
With the above handy, we can prove the following bounds (see Appendix~\ref{ssa:missingeff} for further details). 
%
%\vspace{-0.05in}
\begin{theorem}\label{t:winnowandperceptron}
Let $(G,\by)$ be a labeled graph, $T$ be a random spanning tree of $G$,
$B$   %$B = B_T$ 
be the corresponding BST, and $\Psi_B$ be the (transposed) incidence matrix associated
with $B$.
%\vspace{-0.05in}
\begin{enumerate}
\item If we run Matrix Winnow with similarity instances constructed from $\Psi_B$
(see Algorithm \ref{a:2}) then
the expected number of mistakes $\E[M]$ on $G$ satisfies~\(
\E[M] = \cO\left(\rwcs\,\log^3 n\right),
\)
%\vspace{-0.25in}
\item and if we run the Matrix Perceptron algorithm with similarity instances constructed from $\Psi_B$ then\
\(
\E[M] = \cO\left(\rwcs^2\,\log^4 n\right)~,
\)
%\vspace{-0.05in}
\end{enumerate}
where we denote the resistance-weighted cut-size as $\rwcs =\E[|\mct^T|]= \sum_{(i,j)\in\Phi^G} \er^G_{i,j}$.
\end{theorem}
%\vspace{-0.05in}
%
The bound for Matrix Winnow is optimal up to a $\log^3 n$ factor --- compare to the lower
bound in Proposition \ref{p:boundsimilarity}. However, this tight bound is obtained at the cost 
of having an algorithm which is $\cO(n^3)$ per round, even when run on a tree. This is because
matrix exponentials require storing and
updating a full SVD of the algorithm's weight matrix at each round, thereby 
making this algorithm highly impractical when $G$ is large. 
On the other hand, the Perceptron bound is significantly suboptimal (due to its dependence on the squared
resistance-weighted cut-size), but it has the invaluable advantage of lending itself to a very efficient implementation: 
Whereas a naive implementation would lead to an $\cO(n^2)$ running time per round, we now show that 
a more involved implementation exists which takes only $\cO(\log^2 n)$, yielding an exponential
improvement in the per-round running time.

%\vspace{-0.1in}
\subsection{Implementing Matrix Perceptron on BST}
%\vspace{-0.05in}
%
The algorithm operates on the BST $B$ 
% having $2n-1$ nodes, where $n$ is the number of nodes in the original graph $G$.
% The algorithm 
by maintaining a $(2n-1)\times(2n-1)$ symmetric matrix $F$ with integer entries initially set to zero.
At time $t$, when receiving the pair of leaves $(i_t, j_t)$, the algorithm constructs $\cP_t$, 
the (unique) path in $B$ connecting $i_t$ to $j_t$. Then the prediction $\hat{y}_{i_t,j_t} \in \{0,1\}$ 
is computed as 
%
%\vspace{-0.05in}
\begin{equation}\label{e:predictionstep}
\hat{y}_{i_t,j_t} =
\begin{cases}
1 & {\mbox{if }} \sum_{\ell,\ell' \in \cP_t} F_{\ell,\ell'} \geq 4\log^2 n,\\
0 & {\mbox{otherwise}}~.
\end{cases}
\end{equation}
%\vspace{-0.05in}
%
Upon receiving label $y_{i_t,j_t}$, the algorithm 
updates $F$ as follows. First of all, the algorithm is mistake driven, so an update
takes place only if $y_{i_t,j_t} \neq \hat{y}_{i_t,j_t}$. Let 
$\cN_t$ be the set of neighbors of the vertices in $\cP_t$, and define 
$\cS_t:=\cN_t\setminus(\cP_t\setminus\{i_t,j_t\})$.
We recursively assign integer tags $f_t(\ell)$ to vertices $\ell \in \cN_t$ as follows:\ \
1.  For all $\ell\in\cP_t$, if $\ell$ is the $s$-th vertex in $\cP_t$ then we set $f_t(\ell)=s$;\ \
2.  For all $\ell\in\cN_t\setminus\cP_t$, let $n_{\ell}$ be the (unique) neighbor of $\ell$ 
that is contained in $\cP_t$. Then we set $f_t(\ell)=f_t(n_{\ell})$.\ \
We then update $F$ on each pair $(\ell,\ell')\in\cS_t^2$ as 
\begin{equation}\label{e:updatestep}
F_{\ell,\ell'}\leftarrow F_{\ell,\ell'} + (2y_{i_t,j_t}-1)\,(f_t(\ell)-f_t(\ell'))^2~.
\end{equation}
Figure \ref{f:table} illustrates the process.
\begin{figure}[t!]
\centering
\vspace{-0.5in}
\includegraphics[width=0.9\textwidth]{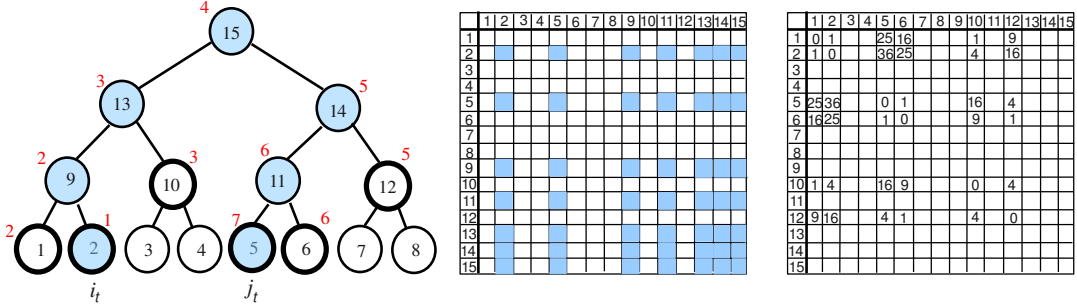}
\vspace{-5.45in}
\caption{{\small Matrix Perceptron algorithm at time $t$ 
with $i_t=2$ and $j_t=5$. 
{\bf Left:} The BST. Light blue vertices are those in $\cP_t$. 
Thick-bordered vertices are those in $\cS_t$. The $f_t$ tags are the red numbers near the 
involved vertices (i.e., those in $\cP_t \cup \cS_t$).
{\bf Middle:} The matrix $F$. In light blue are the entries of $F$ that are summed over 
in (\ref{e:predictionstep}). 
{\bf Right:} The matrix $F$, where numbers are the 
values $(f_t(\ell)-f_t(\ell'))^2$ that are added to ($y_{i_t,j_t}=1, \hat{y}_{i_t,j_t}=0$) 
or subtracted from ($y_{i_t,j_t}=0, \hat{y}_{i_t,j_t}=1$) the respective components of $F$ during 
the update step (\ref{e:updatestep}).} \label{f:table}}
%\vspace{-0.3in}
\end{figure}
The following theorem is the main technical result of this section. Its involved proof 
is given in Appendix \ref{ssa:missingeff}.
%
%\vspace{-0.05in}
\iffalse %%% Reworded to be SHORTER below
\begin{theorem}\label{t:equivalence}
Let $B$ be a BST of a labeled graph $(G,\by)$ with $|V| = n$.
Then the algorithm described by (\ref{e:predictionstep}) and (\ref{e:updatestep}) is equivalent to
a Matrix Perceptron algorithm run with similarity instances constructed from $\Psi_B$.
Moreover, the algorithm takes $\cO(\log^2 n)$ per round, and
a growing data-structure technique exists that stores a representation of $F$ for which 
the overall initialisation time is only $\cO(n)$ (rather than $\cO(n^2)$).
\end{theorem}
\fi
\begin{theorem}\label{t:equivalence}
Let $B$ be a BST of a labeled graph $(G,\by)$ with $|V| = n$.
Then the algorithm described by (\ref{e:predictionstep}) and (\ref{e:updatestep}) is equivalent to
Matrix Perceptron run with similarity instances constructed from $\Psi_B$.
Moreover, the algorithm takes $\cO(\log^2 n)$ per trial, and
there exists an adaptive representation of $F$ with an initialisation time of only $\cO(n)$ (rather than~$\cO(n^2)$).
\end{theorem}

% !TEX root = main7.tex

%\vspace{-0.15in}
\section{The Unknown Graph Case}\label{s:unknown}
%\vspace{-0.05in}
%
We now consider the case when the graph $G = (V,E)$ is {\em unknown} to the learner beforehand.
The graph structure is thus revealed incrementally as more and more pairs $(i_t,j_t)$ get produced 
by the adversary. 
A reasonable online 
%learning 
protocol that includes progressive graph disclosure is the following.
At the beginning of round $t = 1$ the learner knows nothing about $G$, but the number of vertices $n$ 
--- prior knowledge of $n$ makes presentation easier, but could easily be   removed from this setting.  
\iffalse
At the beginning of round $t = 1$ the learner knows nothing about $G$, but the number of vertices $n$\footnote{prior knowledge of $n$ makes presentation easier, but could be removed from this setting.}.   
\fi
In the generic round $t$, the adversary
presents to the learner both pair $(i_t,j_t) \in V\times V$ and a path within $G$ %(sequence of edges, i.e., sequence of pairs of vertex indices) 
from $i_t$ to $j_t$. 
The learner is then compelled to predict whether or not the two vertices are similar.
Notice that, although the presented path may have cut-edges, there might be alternative paths
in $G$ connecting the two vertices with no cut-edges. The learner need not see them.
The adversary then reveals the similarity label $y_{i_t,j_t}$ in $G$, 
%associated with $(i_t,j_t)$ in $G$, 
and the next round begins.
In this setting, 
the adversary has complete knowledge of $G$, and can decide to produce paths and 
place the cut-edges 
%adaptively.
in an adaptive fashion. 
%
% However, we also make the assumption that the adversary must 
% decide {\em beforehand} to stick to a given spanning tree $T$ of $G$, in that
% the adversary always produces paths contained in $T$.\footnote
% {
% This is not a big loss of generality, since the learner can always build a spanning tree on the 
% fly by merging the paths produced by the adversary, without creating cycles. 
% The analysis contained in this section would then refer to this spanning tree.
% }
% Hence, in this setting, the learner only observes $G$ through the incremental disclosure of $T$. 
%
Notice that, because of the incremental disclosure of $G$, no such constructions 
as $\Psi$-based similarity instances and/or BST, as contained
in Section \ref{sec:hs} and Section \ref{s:efficient}, are immediately applicable.

% Moreover, because the learner gets revealed only $T$ out of $G$, it is reasonable to have the 
% learner's
% prediction performance to depend only on the properties of $T$ (like $T$'s cutsize) within $G$.

As a simple warm-up, consider the case when $G$ is a tree
%\MJH{Consider using $T$ rather than $G$ throughout this paragraph?} 
and the $K$ class label sets of vertices (henceforth called {\em clusters}) correspond to connected components 
of $G$. Figure \ref{f:2} (a) gives an example.
% with three (connected) clusters. 
Since the graph is a tree,
the number of cut-edges equals $K-1$. We can associate with such a tree a linear-threshold
function vector $\bu = (u_1, \ldots, u_{n-1})^{\top} \in \{0,1\}^{n-1}$, where $u_i$
is 1 if and only if the $i$-th edge is a cut-edge. The ordering of edges within $\bu$
can be determined ex-post by first disclosure times. For instance,
if in round $t=1$ the adversary produces pair $(6,4)$ and path 
$6 \rightarrow 3 \rightarrow 1 \rightarrow 4$ (Figure \ref{f:2} (a)), then
edge $(6,3)$ will be the first edge, $(3,1)$ will be the second, and $(1,4)$ will
be the third. Then, if in round $t=2$ the new pair is $(3,5)$ and the associated
path is $3 \rightarrow 1 \rightarrow 5$, the newly revealed edge $(1,5)$
will be the fourth edge within $\bu$. With this ordering in mind, the algorithm
builds at time $t$ the $(n-1)$-dimensional vector $\bx_t = (x_{1,t}, \ldots, x_{n-1,t})^\top
\in \{0,1\}^{n-1}$ corresponding to the path disclosed at time $t$, where $x_{i,t}$ 
is 1 if and only if the $i$-th edge belongs to the path. Now, it is clear that 
% the connectivity label 
$y_{i_t,j_t} = 1$ if $\bu^\top\bx_t \geq 1$,
and $y_{i_t,j_t} = 0$  
%(connected) 
if $\bu^\top\bx_t = 0$. Therefore, this turns out to be
a sparse linear-threshold function learning problem, and a simple application of the 
standard Winnow algorithm
\cite{litt88} leads to an $\cO(K\log n) = \cO\left(|\Phi^G|\log n\right)$
%allows us to conclude that the total number of prediction mistakes
mistake bound obtained
% can be achieved 
by an efficient ($\cO(n)$ time per round) algorithm,
independent of the structural properties of $G$, such as its diameter.

One might wonder if an adaptation of the above procedure
exists which applies to a general graph $G$ by, say, extracting a spanning tree
$T$ out of $G$, and then applying the Winnow algorithm on $T$.
Unfortunately, the answer is negative for at least two reasons. 
First, the above linear-threshold model heavily relies on the fact that
clusters are connected, which need not be the case in our similarity problem.
More critically, even if the clusters are connected in $G$, they need not be connected
in $T$. Figure \ref{f:2} (b)-(c) shows a typical 
example where Winnow applied to a spanning tree fails. 
% Even if the blue cluster in Figure \ref{f:2} (b) is connected in $G$, 
% it gets ``shattered" in (c) by the spanning 
% tree $T$, so that connected componenents in $G$ turn out to be disconnected in $T$.
%
%Thus
Given this state of affairs,
% we are lead to make a little detour and 
we are lead to consider a slightly different representation for pairs of vertices 
and paths.
% which is different from previous sections.
Yet, as before, this representation will suggest a linear separability 
condition, 
%over matrices, 
as well as the deployment of appropriate linear-threshold algorithms.
\begin{figure}[t!]
\begin{picture}(-60,200)(-60,200)
\scalebox{0.55}{\includegraphics{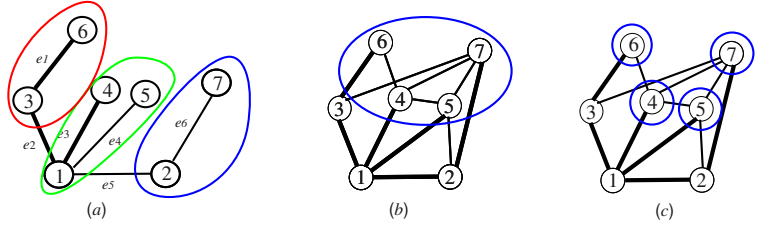}}
\end{picture}
\vspace{-2.0in}
\caption{{\small {\bf (a)} A tree with 3 clusters corresponding to the 3 depicted
connected components. Edges $e_2$ and $e_5$ are the cut-edges. Assuming edges are
initially revealed in the order of their subscripts, the associated
vector $\bu$ is $\bu = (0,1,0,0,1,0)^\top$. The algorithm receives
at time $t = 1$ the pair (6,4) along with path 
$6 \rightarrow 3 \rightarrow 1 \rightarrow 4$ (so that the 3 thick edges are revealed
in the first round). 
The associated feature vector
is $\bx_1 = (1,1,1,0,0,0)^\top$. Vertices $6$ and $4$ are disconnected as $\bu^\top\bx_1 \geq 1$.
{\bf (b)-(c)} The connectivity structure induced by the 
%selected spanning tree (thick edges) 
thick-edged spanning tree
on the blue cut in (b).
Vertices $4$, $5$, $6$, and $7$ are all connected in $G$ under the blue cut (b), but are they
are {\em all disconnected} in $T$ (c).
}
%\vspace{-0.35in}
\label{f:2}
}
\vspace{-0.15in}
\end{figure}
\vspace{-0.1in}
\subsection{Algorithm and analysis}\label{ss:unknownaa}
\vspace{-0.05in}
Algorithm \ref{a:1} contains the pseudocode of our algorithm.
%online similarity algorithm. 
When interpreted as operating on vectors, the algorithm is simply
an $r$-norm perceptron algorithm~\cite{gls97,g03} with nonzero threshold, and norm
$r = 2\log (n-1)^2 =  4\log (n-1)$, being $(n-1)^2$ the length of the vectors maintained throughout,
and $s$ the dual to norm $r$. 
At time $t$, the algorithm observes pair $(i_t,j_t)$ and path $p_{(i_t\rightarrow j_t)}$,
% \footnote
% {
% Notice that the naming of edges determines the order of $\bp_t$'s components, while
% the traversal direction of edges determines the sign of $\bp_t$'s components. 
% The algorithm names the edges in order of their initial occurrence, directs them 
% from low index to high index node, and stays consistent with this convention
% throughout. 
%In other words, the signs of $\bx_t$'s components are not an extra information given to the algorithm.
% }
builds the instance vector $\bx_t \in \{-1,0,1\}^{n-1}$ and
%= Q\bp_t \in $, 
the long vector $\vct(X_t)$ out of the rank-one matrix 
$X_t = \bx_t\bx_t^\top$,
where $\vct(\cdot)$ is the standard vectorization of a matrix that stacks its columns one underneath the other.
In order to construct $\bx_t$ from $p_{(i_t\rightarrow j_t)}$,
the algorithm maintains a forest %of trees 
made up
of the union of paths seen so far. If pair $(i_t,j_t)$ is already connected by a path
$p$ in the current forest, then $\bx_t$ is the instance vector associated with path $p$
(as for the Winnow algorithm on a tree in the previous section, but taking edge orientations
%\MJH{No longer explained in the main body.  Possibly either explain or reference appendix} 
into account -- see Figure \ref{f:3} in Appendix \ref{ssa:unknown} for details).
%see Figure \ref{f:3}).
Otherwise, path $p_{(i_t\rightarrow j_t)}$ is added to the forest and $\bx_t$ will
be an instance vector associated with the new path $p_{(i_t\rightarrow j_t)}$.
In adding the new path to the forest, we need to make sure that no circuits 
are generated. 
%along the way. 
In particular, as soon as a revealed edge in a path
causes two subtrees to join, the algorithm merges the two subtrees and processes
all remaining edges in that path in a sequential manner so as to avoid 
generating circuits. 
% Figure \ref{f:3} in Appendix \ref{ssa:unknown} illustrates the process
% by means of an example. 
%
% 
%Notice that 
The algorithm will end up using a spanning tree $T$ of $G$ for building
its instance vectors $\bx_t$. This spanning tree is determined on the fly
by the adversarial choices of pairs and paths, 
% $(i_t,j_t)$ and the associated paths,
so {\em it is not known to the algorithm} ahead of time.\footnote
{
In fact, because the algorithm is deterministic, this spanning tree is fully 
determined by the adversary. We are currently exploring to what extent randomization
is beneficial for an algorithm in this setting.
} 
But 
%the reader should observe that 
any later change to the spanning forest
is designed so as to keep consistency with all previous 
%instance 
vectors $\bx_t$.
%produced by the algorithm. 
% So, in a sense, the algorithm is operating as if it knew this spanning tree beforehand. 

% Moreover, since the spanning tree $T$ is unknown to the algorithm, so is matrix $Q$.
% Once again, one should observe that the algorithm need not explicitly construct $Q$. 
% It is the instance vector $\bx_t$ that contains all aspects of $Q$ that are needed 
% to the algorithm in round $t$.
%
The decision threshold $ (r-1)||\bx_t||^4_r = (r-1)||\vct(X_t)||^2_r$ 
% in the algorithm's pseudocode
follows from a standard analysis of the $r$-norm perceptron algorithm
with nonzero threshold (easily adapted from \cite{gls97,g03}), as well as the update 
rule. 
%on mistaken rounds. 
In short, since the graph is initially unknown, the algorithm is pretending to learn vectors 
rather than (Laplacian-regularized) matrices, and relies on a regularization 
that takes advantage of the sparsity of such vectors. 
% This sparsity is insured by the sparsity
% of the coefficient vectors $\bu_k$ associated with the class labels w.r.t. $T$.
%
%% ----------------------------------------------------------------------------
%\hspace{-0.4in}
\begin{algorithm2e}[t]
\SetKwSty{textrm}
\SetKwFor{For}{{\bf For}}{}{}
\SetKwIF{If}{ElseIf}{Else}{if}{}{else if}{else}{}
\SetKwFor{While}{while}{}{}
%\vspace{0.1cm}
{\bf Input:}~ Number of vertices $n = |V|$, $V = \{1,\ldots,n\}$, $n \geq 3$;\\
{\bf Initialization :} $\bw_{0} = \bzero \in \R^{(n-1)^2}$;
$r = 4\log(n-1)$, $s = \frac{r}{r-1}$;\\
\For{ $t=1,2,\ldots, T$ :}
{
\vspace{-0.15in}
{
\begin{itemize}
\item Get pair of vertices $(i_t,j_t) \in V^2$, and path $p_{(i_t\rightarrow j_t)}$. 
Construct instance vector $\bx_t \in \{-1,0,1\}^{n-1}$ as explained in the main text;
\vspace{-0.05in}
\item Build $(n-1)^2$-dimensional vector $\vct(X_t)$, where $X_t = \bx_t\bx_t^\top$, 
and predict ${\hat y}_t \in \{0,1\}$ as
\vspace{-0.05in}
\[
{\hat y_t} =
\begin{cases}
1 &{\mbox{if $\bw_{t-1}^\top \vct(X_t) \geq (r-1)\,||\bx_t||^4_r$ }}\\
0 &{\mbox{otherwise}},
\end{cases}
\]
\vspace{-0.2in}
\item Observe $y_t \in \{0,1\}$ and, if mistake ($y_t \neq {\hat y_t}$), update
\vspace{-0.1in}
\[
f(\bw_t) \leftarrow f(\bw_{t-1}) + (y_t - {\hat y_t})\,\vct(X_t),\qquad 
                                                 f(\bw) = \nabla ||\bw||^2_{s}/2~.
\]
\vspace{-0.3in}
\end{itemize}
}
}
\vspace{-0.33in}
\caption{$r$-norm Perceptron for similarity prediction in unknown graphs.}
\label{a:1}
\end{algorithm2e}
%% ----------------------------------------------------------------------------
%
% The next theorem shows that the above encoding of paths makes the 
% problem linearly separable through sparse separators, whose degree of sparsity is
% related to the cutsize of the graph $G$ restricted to the spanning tree $T$.
% The 
% We have the following result, whose analysis relies 
The analysis of Theorem \ref{t:unknown} below rests
on ancillary (and classical) properties of 
matroids on graphs. These are recalled in Appendix \ref{ssa:unknown}, before
the proof of the theorem.  
%
% on the standard analysis of the $r$-norm Perceptron algorithms. 
%
{\sloppypar
\vspace{-0.07in}
\begin{theorem}\label{t:unknown}
With the notation introduced in this section, 
let Algorithm \ref{a:1} be run on an arbitrary sequence of pairs 
$(i_1,j_1), (i_2,j_2), \ldots$ and associated sequence of paths
$p_{(i_1\rightarrow j_1)},p_{(i_2\rightarrow j_2)}, \ldots $\,.
% Let $T$ be the spanning tree generated by the adversary through Algorithm \ref{a:1} 
% at the end of the game. 
%Then the number of prediction mistakes $M$ satisfies
Then we have the mistake bound
\(
M 
%= O\left(\left(\sum_{k=1}^K |\Phi_k^T|^2\right)^2\,\log n \right)
%= \cO\left(\left(\sum_{k=1}^K |\Phi_k^G|^2\right)^2\,\log n \right)\,.
= \cO\left(|\Phi^G|^4\,\log n \right)\,.
\)
% where $K$ is the total number of classes in $G$, and 
% $|\Phi_k^G|$ is the cutsize of the $k$-th class.
%, restricted to tree $T$. 
\end{theorem}
\vspace{-0.05in}
}

\vspace{-0.1in}
\begin{remark} 
As explained in the proof of Theorem \ref{t:unknown}, 
the separability condition (\ref{e:separ}) 
%(see Appendix \ref{ssa:unknown}) 
allows one to run any vector or matrix mirror descent linear-threshold algorithm. 
In particular,
since matrix $U$ therein is spectrally sparse (rank $K << n$), one could use unitarily 
invariant regularization methods, like (squared) trace norm-based online algorithms 
(e.g., \cite{w07,ccg10,kst12}). %As we said, 
For instance, Matrix Winnow (more generally, Matrix EG-like algorithms~\cite{matrixeg}) 
% seems to perform best among them for our specific problem. 
% This algorithm 
would get bounds which are linear in the cutsize but also (due to their unitary invariance) 
linear in $||\bx_t||_2^2$. The latter can be as large as the diameter of $T$, which
can easily be $\cO(n)$ even if the diameter of $G$ is much smaller. This makes these bounds
significantly worse than Theorem \ref{t:unknown} when the total cutsize 
$|\Phi^G|$ is small compared to $n$ (which is our underlying assumption throughout).
Group norm regularizers can also be used. Yet,
because $X_t$ has rank one, when $|\Phi^G|$ is small
these regularizers do not lead to better bounds\footnote
{
In fact, our $r$-norm algorithm operating on $\vct(\cdot)$ vectors is equivalent
to a group norm regularization technique applied to the corresponding matrices where 
row and column norms are both equal to $s=\frac{r}{r-1}$.
} 
than Theorem \ref{t:unknown}.
% Finally, 
Moreover, it is worth mentioning that, among the standard mirror descent linear-threshold
algorithms operating on vectors $\vct(\cdot)$, our choice of the $r$-norm Perceptron is motivated 
by the fact this algorithm achieves a logarithmic bound in $n$ with no prior knowledge of the actual
cutsize $|\Phi^G|$ (or an upper bound thereof) -- see Section \ref{sec:hs}, and the discussion
in \cite{g03} about tuning of parameters in $r$-norm Perceptron and Winnow/Weighted Majority-like
algorithms.
%\end{remark}
%
%\begin{remark}
%Unlike matrix Winnow (Section \ref{s:..}), 

As a final remark, 
% on running time,
our algorithm has an $\cO(n^2)$ running time per round, trivially due to the update rule operating on $\cO(n^2)$-long vectors. 
%Observe that 
The construction of instance vector $\bx_t$ out of path
$p_{(i_t\rightarrow j_t)}$ can indeed be implemented faster than $\Theta(n^2)$ by maintaining well-known data 
structures for disjoint sets (e.g., \cite[Ch. 22]{clr90}). 
%
% During this construction, deciding 
% whether a given edge $(i,j)$ 
% has to be added to the forest just amounts to checking if $i$ and $j$ belong to the same set.
% Moreover, producing an alternative path to the one presented by the adversary 
% (recall Figure \ref{f:3}) can be done in $\cO(n)$ time through a breadth-first visit.
\end{remark}
\vspace{-0.1in}

\iffalse
************************************************************************
&=  \E_T \left[\sum_{i_1,i_2, \ldots, i_{|q|}\,:\ i_1 + i_2 + \ldots i_{|q|} = 4} 
\frac{4!}{i_1!\,i_2!\,\ldots\, i_{|q|}!} 
\prod_{j = 1\ldots |q|\,:\, i_j \geq  1}X_j \right]\\
&=  \sum_{i_1,i_2, \ldots, i_{|q|}\,:\ i_1 + i_2 + \ldots i_{|q|} = 4} 
\frac{4!}{i_1!\,i_2!\,\ldots\, i_{|q|}!} 
\E_T \left[\prod_{j = 1\ldots |q|\,:\, i_j \geq  1} X_j \right]\\
&\leq  \sum_{i_1,i_2, \ldots, i_{|q|}\,:\ i_1 + i_2 + \ldots i_{|q|} = 4} 
\frac{4!}{i_1!\,i_2!\,\ldots\, i_{|q|}!} 
\prod_{j = 1\ldots |q|\,:\, i_j \geq  1} \E_T X_j
*************************************************************************
\fi

%\input{conclusions1}

%\bibliographystyle{mlapa}

\appendix

\section{Proofs}
This appendix contains all omitted proofs. Notation is as in the main text.

% !TEX root = main7.tex
\subsection{Missing proofs from Section~\ref{sec:corr}}\label{sec:proofofreductions}
\sloppypar{
The set of example sequences {\em consistent} with a concept $f$ for class prediction is denoted by 
$\cseq^c(f) := (\{(x,f(x))\}_{x\in \cX})^*$, and for similarity prediction by 
$\cseq^s(f) := (\{((x',x''),\szm(f(x'),f(x''))\}_{x',x''\in \cX})^*$.  
A prediction algorithm is a mapping $A: \seqset{(\cX\times \cY)} \into \cY^\cX$ from example sequences to prediction functions.  
Thus if $A$ is a prediction algorithm and $S=\seq{(x_1,y_1),\ldots,(x_{T},y_{T})}\in  \seqset{(\cX\times \cY)}$ is an example sequence, then  the {\em online} prediction mistakes are 
\[ % better inline if we can prevent line overflow
M_A(S) := \sum_{t=1}^{T} \id{A(\seq{(x_1,y_1),\ldots,(x_{t-1},y_{t-1})})(x_t) \ne y_t}\,.
\]
We write $S' \subseteq S$ to denote that $S'$ is a subset of $S$ as well to denote that $S'$ is a {\em subsequence} of $S$. 

We now  introduce a weaker notion of a mistake bound as defined with respect to  specific sequences rather than to the alternate notion of a concept.  The weakness of this definition allows the construction in the following lemma to apply to ``noisy'' as well ``consistent'' example sequences.
}
\begin{definition}\label{def:smb}
Given an algorithm $A$, we define the {\bf subsequential mistake bound} with respect to an example sequence $S$ as
\[\smb_A(S) := \max_{S' \subseteq S} M_A(S')~.
\]
\end{definition}
Thus, a subsequential mistake bound is simply the ``worst-case'' mistake bound over all subsequences.
\begin{lemma}\label{lem:equiv}
Given an  online {\em classification} algorithm $A$, there exists a similarity algorithm $A'$ such that for 
every  sequence $S= \seq{(x_1,y_1),\ldots,(x_{2T},y_{2T})}$ its mistakes on  
\[ S' = \seq{((x_1,x_2),\szm(y_1,y_2)),\ldots,((x_{2T-1},x_{2T}),\szm(y_{2T-1},y_{2T}))}\] is bounded as
\begin{equation}\label{eq:mup}
M_{A'}(S') ) \le  c\, \smb_{A}(S) \log_2 K\,,
\end{equation}
with $c< 5$.
\end{lemma}
\begin{proof}
The proof of~\eqref{eq:mup} works by a modification of the standard weighed majority algorithm~\cite{\WM} arguments.  
The key idea is that similarity reduces to classification if we received the actual class labels as feedback rather than just {\sc{similar/dissimilar}}\ as feedback.  Since we do not have the class-labels, 
we instead ``hallucinate'' all possible feedback histories and then combine these histories on each trial using a weighted majority vote.   If we only keep track of histories generated when the weighted majority vote is mistaken, the bound is small.  Our master voting  algorithm $A'$ follows.   
\begin{enumerate}
\item {\bf Initialisation:}  We initialize the parameter $\beta = 0.294$.  We create a pool containing example sequences (``hallucinated {\em histories}'')  $\cS := \{s\}$, with initially the empty history $s =\langle\rangle$ with weight $w_s :=1$ .
\item {\bf For} $t=1,\ldots,T$ {\bf do} \label{en:re}
\item {\bf Receive:} the pattern pair $(x_{2t-1},x_{2t})$ 
\item {\bf Predict:} {\sc similar} if 
\[ \sum_{s\in \cS} w_s \id{A(s)(x_{2t-1}) = A(s)(x_{2t})} \ge \sum_{s\in \cS} w_s \id{A(s)(x_{2t-1}) \ne A(s)(x_{2t})} \]
otherwise predict {\sc dissimilar}.
\item {\bf Receive:} Similarity feedback $\szm(y_{2t-1},y_{2t})$ if prediction was correct go to~\ref{en:re}.  
\item Two cases, first if this algorithm predicted {\sc similar} when the pair was {\sc dissimilar} then 
for each history $s\in\cS$ with a mistaken prediction create $K\times (K-1)$ histories $s_{1,2},\ldots,s_{K,K-1}$ so that $s_{i,j}$ is equal to $s$ but has the two ``speculated'' examples $(x_{2t-1},i),(x_{2t},j)$ appended to it.   Then set $w_{s_{1,2}} = \ldots = w_{s_{K,K-1}} := \frac{\beta}{K(K-1)}w_s$ and remove $s$ from $\cS$.  Second (predicted {\sc dissimilar}) as above but now we need to  add only $K$ new histories to  the pool.
\item Go to~\ref{en:re}.    
\end{enumerate}
Observe that there exists a history $s^*\in\cS$  generated by no more  $\smb_{A}(S)$ ``mistakes'' since there is always at least one history in the pool $\cS$ which is a subsequence of $S$.
%\MJH{Please let me know if I should elaborate.} 
Thus 
\[ 
w_{s^*} \ge \left({\frac{\beta}{K(K-1)}}\right)^{\smb_{A}(S)}~. 
\]  
Furthermore, the total weight of the pool of histories $W := \sum_{s\in\cS} w_s$  is reduced to a fraction of its 
weight no larger than $\frac{1+\beta}{2}$ whenever this master algorithm $A'$ makes a mistake.  
Thus since $W\ge w_{s^*}$, we have
\[
\left(\frac{1+\beta}{2}\right)^{M_{A'}(S')}  \ge \left({\frac{\beta}{K(K-1)}}\right)^{\smb_{A}(S)}.
\]
%
%where $M$ is the number of mistakes of $W(A)$.  
Solving for $M_{A'}(S') $ we have
\[
M_{A'}(S') \le \smb_{A}(S) \left(\frac{\log_2\frac{K(K-1)}{\beta} }{\log_2 \frac{2}{1+\beta}}\right)\,.
\]
Substituting in $\beta = .294$ allows us to obtain the upper bound of~\eqref{eq:mup} with $c\approx4.99$.  
\end{proof}
We observe, that reduction of similarity to classification holds for a wide variety online mistake bounds.  Thus, e.g., we do not require the input sequence to be {\em consistent}, i.e., in $S$ we may have examples $(x',y')$ and $(x'',y'')$ 
such that $x' = x''$ but $y' \ne y''$.  The usual type of mistake bound is  {\em permutation invariant} i.e., 
the bound is the ``same'' for all permutations of the input sequence; typical examples include the  the 
{\sc Weighted Majority}~\cite{\WM} and {\sc $p$-Norm Perceptron}~\cite{gls97,g03} algorithms.  
Observe that if ${B^p_{A}}(S)$ is a permutation invariant bound, then  $\smb_{A}(S) \le B^p_{A}(S)$, 
since every subsequence of $S$ is the prefix of a permutation of $S$.  
However, our reduction also more broadly applies to such ``order-dependent'' bounds, 
as the shifting-expert bounds in~\cite{\trackingexpert}.   

We now show that classification reduces to similarity.  This reduction is efficient and does not introduce a multiplicative constant but requires  the stronger assumption of {\em consistency} not required by Lemma~\ref{lem:equiv}.  
%In the following theorem we now require that our example sequence corresponds to some permutation of a consistent labeling.  
%
\begin{lemma}\label{lem:simtoclass}
Given an online similarity algorithm $A^s$ there exists an online classification algorithm $A^c$ such that for 
any concept $f$ if $S\in \cseq^c(f)$ then
\begin{equation}\label{eq:simtoclass}
M_{A^c}(S) \le \max_{S' \in \cseq^s(f)} M_{A^s}(S') +K~.
\end{equation}
\end{lemma}
\begin{proof}
As a warm-up, pretend we know a set $P\subseteq X$ such that $|P| = K$ 
and for each $i\in \{1,\ldots, K\}$ there exists an $x\in P$ such that $f(x) = i$. 
Using $P$ we create algorithm $A^c$ as follows. 
We maintain a history (example sequence) $h$, which is initially empty.  
Then on every trial when we receive a pattern $x_t$ we predict 
$\hat{y}_t \in \{f(x) : A^s(h)((x,x_t)) = \text{ {\sc similar}}, x\in P\}$, 
and if the set contains multiple elements or is empty then we predict arbitrarily.  
If $A^c$ incurs a mistake, we add to our history $h$ the $K$ examples  
$((x,x_t),\szm(f(x),f(x_t)))_{x\in P}$.  
Observe that if $A^c$ incurs a mistake then at least one example corresponding to a mistaken 
similarity prediction is added to $h$ and necessarily $h\in \cseq^s(f)$. 
Thus the mistakes of $A^c$ are bounded by 
$\max_{S' \in \cseq^s(f)} M_{A^s}(S')$.  
Now, since we do not actually know a set $P$, we may modify our algorithm $A^c$ 
so that although $P$ is initially empty we predict as before, and if we make a mistake on $x_t$ 
because there does not exist an $x\in P$ such that $f(x) = f(x_t)$, 
we then add $x_t$ to $P$. We can only make $K$ such mistakes, so we have the bound of~\eqref{eq:simtoclass}. 
\end{proof}

\begin{proofof}{Theorem~\ref{thm:doubleequiv}}
If $S$ is a classification sequence consistent with a concept $f$ and $A$ is a classification algorithm, 
then $\smb_A(S) \le \mb_A(f)$, and hence~\eqref{eq:mup} implies~\eqref{eq:dmup}. 
Then, since~\eqref{eq:simtoclass} is equivalent to~\eqref{eq:dsimtoclass}, we are done.
\end{proofof}

\subsubsection{The $\log K$ term is necessary in Theorem~\ref{thm:doubleequiv}}\label{sec:lognec}
In our study of class prediction on graphs we observed (see Appendix~\ref{ssa:missingprospectus}) that certain 2-class bounds may be converted to $K$-class bounds with no explicit dependence on $K$.   Yet Theorem~\ref{thm:doubleequiv} introduces a factor of $\log K$ for similarity prediction.
%In our study of similarity prediction over graphs in Section~\ref{sec:spg}, 
%we give a number bounds that do not introduce a ``$\log K$'' factor from classification to similarity.
% \MJH{The only bounds that do not introduce a factor of Log K or larger factor seem be either 
% the non-tree Laplacian technique 
% on small-diameter graphs, or tree-winnow, or the inefficient version space argument under the 
% ``connectivity'' assumption.}  
So a question that arises is this simply a byproduct of the above analysis or is the 
``$\log K$'' factor tight.   
In the following, we demonstrate it is tight by introducing the {\sc paired permutation} problem, 
which may be ``solved'' in the classification setting with no more than $\cO(K)$ mistakes. 
Conversely, we show that an adversary can force $\Omega(K \log K)$ mistakes in the similarity 
setting.

We introduce the following notation.  
Let $z : \{1,\ldots,K\} \into \{1,\ldots,K\}$ 
denote a {\em permutation} function, a member of the set $\ZZ$ of all $K!$ bijective functions 
from $\{1,\ldots,K\}$ to $\{1,\ldots,K\}$.   
A {\em paired permutation} function is the mapping $y_z  : \{1,\ldots,K\}^2 \into \{1,\ldots,K\}$, 
with $y_z(x',x'') := \max(z(x'),z(x''))$.  
So, for example, consider a 3-element permutation 
$z(1)\rightarrow 2,  z(2)\rightarrow 3$, and $z(3) \rightarrow 1$.  Then, e.g., 
$y_z(1,2) \rightarrow 3$ and $y_z(1,1) \rightarrow 2$.  
Thus, we define the set of the {\sc paired permutation} problem example sequences 
for class prediction as $\cpp^c := \cup_{z\in \ZZ} \cseq^c(y_z)$, and for similarity as 
$\cpp^s := \cup_{z\in \ZZ} \cseq^s(y_z)$.
\begin{theorem}\label{thm:lognec}
There exists a class prediction algorithm $A$ such that for any $S\in \cpp^c$ 
we have $M_A(S) = \cO(K)$.  
Furthermore, for any similarity prediction algorithm $A'$, there exists an $S'\in \cpp^s$ 
such that $M_{A'}(S') = \Omega(K \log K)$.  
\end{theorem}
\begin{proofof}{Theorem~\ref{thm:lognec}}
First, we show that there exists a class prediction algorithm $A$ such that for any $S\in \cpp^c$ we have 
$M_{A}(S) = \cO(K)$.  Consider the simpler problem for the concept class of permutations $\cup_{z\in \ZZ} \cseq^c(z)$.  
By simply predicting {\em consistently} with the past examples we cannot incur more than $K-1$ mistakes. 
The algorithm $A_0(s)$ ({\sc consistent predictor}), predicts $y$ on receipt of pattern $x_t$ 
if there exists some example $(x,y)$ in its history $s$ such that $x_t = x$, 
otherwise it predicts a $y\in\{1,\ldots,K\}$ not in its history.
Now, using $A_0$ 
%({\sc consistent predictor}) 
as a base algorithm, we can use the principle of the master algorithm of 
Lemma~\ref{lem:equiv} to achieve $\cO(K)$ mistakes for the paired permutation problem. 
Thus, when we receive a pair $((x',x''),y)$ either $y_z(x')=y$ or $y_z(x'')=y$, 
hence on a mistake we may ``hallucinate'' these two possible continuations.   
Our {\em class} prediction example sequence is
$S= \seq{(x'_{1},x''_{1}),y_1),\ldots,((x'_{T},x''_{T}),y_T)}\in \cpp^c$, 
and the algorithm $A$ follows.
\begin{enumerate}
\item {\bf Initialisation:}  We initialize the parameter $\beta = 0.294$.  
We create a pool containing example sequences (``hallucinated {\em histories}'')  
$\cS := \{s\}$, with initially the empty history $s =\langle\rangle$ with weight $w_s :=1$ .
\item {\bf For} $t=1,\ldots,T$ {\bf do} \label{en:pp}
\item {\bf Receive:} the pattern $x_t=(x'_{t},x''_{t})$ 
\item {\bf Predict:} 
\begin{equation}\label{eq:pphat}
\hat{y}_t = \argmax_{k\in \{1,\ldots,K\}} \sum_{s\in \cS} w_s \id{\max(A_0(s)(x'_{t}),A_0(s)(x''_{t}))=k}\,.
\end{equation}
\item {\bf Receive:} Class feedback $y_t\in\{1,\ldots,K\}$. If prediction was correct go to~\ref{en:pp}.  
\item For each history $s\in\cS$ with a mistaken prediction, 
create two histories $s',s''$ so that $s'$ ($s''$) is equal to $s$ but has the 
example $(x'_{t},y)$ (the example $(x''_{t},y)$) appended to it.   
Then set $w_{s'} =  w_{s''} := \frac{\beta}{2}w_s$~, and remove $s$ from $\cS$.  
\item Go to~\ref{en:re}.    
\end{enumerate}
Observe that there exists a history $s^*\in\cS$ generated by no more  $K-1$ ``mistakes''.  
This is because, by induction, there is always a consistent history 
(i.e., the empty history is initially consistent, 
and when the master algorithm $A$ makes a mistake and a consistent history makes a mistake, then
either the continuation $(x'_{t},y)$ or $(x''_{t},y)$ is consistent).  
Finally, observe that once a consistent history contains $K-1$ examples, it can no longer make mistakes. 
Thus 
\[ 
w_{s^*} \ge \left({\frac{\beta}{2}}\right)^{K-1}. 
\]  
Furthermore, the total weight of the pool of histories $W := \sum_{s\in\cS} w_s$  
is reduced to a fraction of its weight no larger than $\frac{1+\beta}{2}$ 
whenever this master algorithm $A$ makes a mistake. 
%\MJH{Could add the phrase (as the correct label can account for at most $1/2$ of the weight  in $W$ if there incorrect 
% prediction independent of K).}  
Since $W\ge w_{s^*}$, we have
\[
\left(\frac{1+\beta}{2}\right)^{M_{A}(S)}  \ge \left({\frac{\beta}{2}}\right)^{K-1}~.
\]
%
%where $M$ is the number of mistakes of $W(A)$.  
Solving for $M_{A}(S) $ we can write
\[
M_{A}(S) \le (K-1) \left(\frac{\log_2\frac{2}{\beta} }{\log_2 \frac{2}{1+\beta}}\right)\,.
\]
Substituting in $\beta = 0.294$ allows us to obtain the upper bound of 
$M_{A}(S) \le 4.1 (K-1)$. The argument then follows as in Theorem~\ref{thm:doubleequiv}.   

Now consider  the similarity problem.  
If we receive an instance of the form $(((x',x''),(x'',x'')),y)$ 
(with $x'\ne x''$) then $y =\text{\sc similar}$ implies $z(x') < z(x'')$ 
and $y =\text{\sc disimilar}$ implies $z(x') > z(x'')$.  
Thus with each mistaken example we learn precisely a single `$<$' comparison. 
It follows from standard lower bounds on comparison-based sorting algorithms~(e.g.,~\cite{clr90}) 
that an adversary can force $\Omega( K \log K)$ comparisons, and thus mistakes, 
for any ``comparison''-algorithm to learn an arbitrary permutation.
%\MJH{I can detail either direction as necessary.}
\end{proofof}

Any problem associated with a set of example sequences $\cS$ may be iterated into a set of $r$ 
independent problems by a cross-product-like construction, so that if $S^1,\ldots,S^r \in \cS$ and 
if $S^i = \seq{(x^i_1,y^i_1),\ldots,(x^i_{T_i},y^i_{T_i})}$ then an $r$-iterated example sequence is 
\[
\seq{((x^1_1,1),y^1_1),\ldots,((x^1_{T_1},1),y^1_{T_1}),\ldots,((x^i_1,i),y^i_1),\ldots,((x^r_{T_r},r),y^r_{T_r})}\,.
\]
We have simply conjoined the $r$ example sequences into a single example sequence with each pattern 
``$x$'' paired with an integer indicating from which sequence it originated. 
Thus by $r$-iterating the paired permutation problem we trivially observe mistake bounds of 
$\cO(r K)$ and $\Omega(r K \log K)$ for all $r\in \N$  in the class and similarity setting, respectively,
thereby implying that the multiplicative ``$\log K$'' gap occurs for an infinite family
of classification/similarity problems.

% !TEX root = main7.tex 

\subsection{Missing proofs from Section~\ref{sec:spg}} \label{ssa:missingprospectus}

\subsubsection*{Lifting $2$-class prediction to $K$-class prediction on graphs}
Suppose we have an algorithm for the 2-class graph labeling problem with
a mistake bound of the form $M \le c |\mct^{G}(\by)|$ for all $\by \in \{1,2\}^n$, 
with $c\ge 0$, We show that this implies the existence in the $K$-class 
setting of an algorithm with a bound of $M \le 2 c |\mct^{G}(\by)|$ for all 
$\by \in \{1,\ldots, K\}^n$, where $K$ need not be known in advance to the algorithm.  
%
\iffalse 
*************************************************
\begin{equation}\label{eq:ba}
\operatorname{Mistakes} \le  f(\mct^{G}(\bu)),  
\end{equation} 
for all $\bu\in\N^n_2$ this implies that there exists multi-class setting an algorithm 
with a bound
\begin{equation}\label{eq:tba}
\operatorname{Mistakes} \le f(2 \mct^{G}(\bu)),  
\end{equation}
**************************************************
\fi  
%

The algorithm simply works by combining the predictions of ``one versus rest'' classifiers. 
We train one classifier per class, and introduce a new classifier as soon as that class first appears.  
On any given trial, the combination is straightforward: If there is only one classifier
predicting with its own class then we go with that class,
otherwise we just assume a mistake. Thus, on any given trial, we can only be mistaken if one of the 
current ``one-verse-rest'' classifiers makes a mistake. This implies that our mistake bound is the sum 
of the mistake bounds of all of the ``one-verse-rest'' classifiers. Because each such binary classifier
has a mistake bound of the form $M \le c |\mct_k^{G}(\by)|$, and 
$\sum_{k=1}^K |\mct_k^{G}(\by)| = 2 |\mct^{G}(\by)|$, we have that the $K$-class classifier
has a bound of the form $2c|\mct^{G}(\by)|$.

%\subsection*{Missing proofs from Section~\ref{sec:echard}}
\iffalse
The proof follows follows from reducing Ising model 
In~\cite[Theorem 15]{JS93}, it was shown that computing the 
partition function for the ferromagnetic (attractive) Ising model on a generic 
graph is \#P-complete. 
\fi

\iffalse
\begin{figure}
\begin{quote}
{\sc Instance :} An $n$-vertex graph $G$, and a natural number, $\beta$, presented in unary notation.  \\

{\sc Output:} The value of the partition function $Z_P(\beta)$,
\begin{equation}
Z_G(\beta) :=  \sum_{\by\in\{1,2\}^n}2^{-\beta|\cut^G(\by)|}
\end{equation}
\end{quote}\caption{Partition function problem for the Ising Model on a graph}
\end{figure}
\fi

\ \\
\noindent{\bf Proof of Theorem \ref{t:hardness}: }
We show that computing the partition function for the Ising model on a general graph reduces to  computing the partition function problem for the Ising Model on a path graph with pairwise constraints hence showing  \#P-completeness.  
The partition problem for the Ising Model on a graph is defined by, 
\begin{quote}
{\sc Instance :} An $n$-vertex graph $G$, and a natural number, $\beta$, presented in unary notation.  \\

{\sc Output:} The value of the partition function $Z_G(\beta)$,
%\begin{equation}
\[
Z_G(\beta) :=  \sum_{\by\in\{1,2\}^n}2^{-\beta|\cut^G(\by)|}~.
\]
%\end{equation}
\end{quote}
This problem was shown \#P-complete in~\cite[Theorem 15]{JS93}.  
The reduction to the partition problem on a path graph with constraints is as follows.

We are given a graph $G = (V_G,E_G)$ with $n=|V_G|$, and further assume each vertex is ``labeled'' 
uniquely from $1,\ldots, n$. We construct the following path graph with pairwise constraints
(see Figure \ref{f:ising}) for an illustration.
\begin{enumerate}
\item Find a spanning tree $T = (V_T,E_T)$ of $G$, and let $R = E_G-E_T$.
\item Perform a depth-first-visit of $T$.  From the $2n-1$ vertex visit sequence, 
create an isomorphic path graph $P_0$ with $2n-1$ vertices such that each vertex in $P_0$ 
is labeled with the corresponding vertex label from the visit of $T$.  
Thus each edge of $T$ is mapped to two edges in  $P_0$.
\item We now proceed to create a path graph $P = (V_P,E_P)$ from $P_0$, 
which also includes each edge in $R$ twice. We initialize $P$ as a 
``duplicate'' of  $P_0$ including labels. For each edge $(v_r',v_r'')\in R$ we then do the following:
\begin{enumerate}
\item Choose an arbitrary vertex $v'\in V_P$ so that  $v'$ and $v_r'$ have the same label;
\item Let $v''''$ be a {\em neighbor} of $v'$ in $P$ (i.e, $(v',v'''')\in E_P$);
\item Add vertices $v''$ and $v'''$ to $P$ with the labels of $v_r''$ and $v_r'$, respectively;
\item Remove the edge $(v',v'''')$ from $P$ and add the edges $(v',v''),(v'',v''')$ and $(v''',v'''')$ to $P$.
\end{enumerate}
\item Finally create  pairwise equality constraints between all vertices with the same label.
\end{enumerate}
Thus observe for every edge in $G$ there are two analogous edges in $P$, 
and furthermore if edge $(v,w) \not\in G$ then there is not an analogous edge 
in $P$. Hence $Z_G(2\beta) = Z_P(\mathcal{C},\beta)$.

\begin{figure}[t!]
  \centering
  \includegraphics[width=150mm]{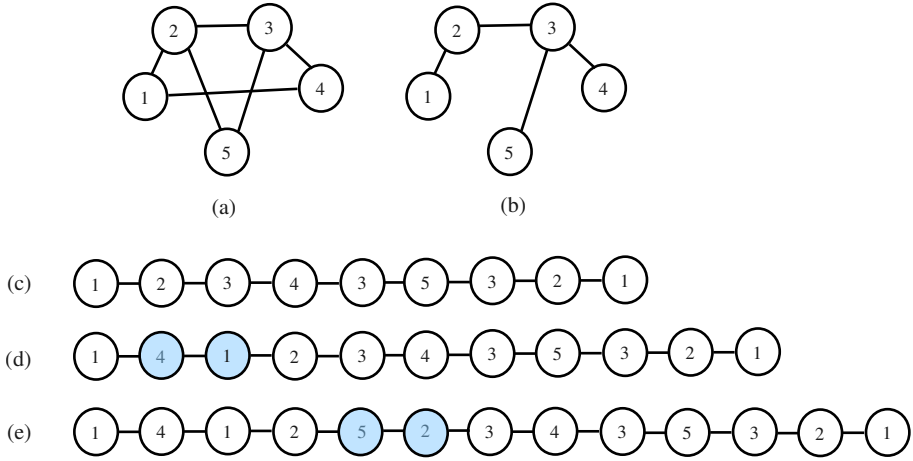}
\vspace{-5in}
\caption{(a) The graph $G$ with vertices labelled $1$ to $5$. (b) A spanning tree $T$ of $G$. Note that $R=\{(1,4),(2,5)\}$.  
(c) The labeled path graph $P_0$.  
(d) Addition of vertices associated with edge (1,4) (the light blue vertices).  
(e) Addition of vertices associated with edge (2,5) (the light blue vertices), 
forming the path graph $P$. Note that every edge in $G$ now has exactly two analogous edges in $P$, 
accounting for all the edges in $P$.
\label{f:ising}}
\end{figure}

%\end{proofof}

% !TEX root = main7.tex

\ \\
\noindent{\bf Proof sketch Proposition~\ref{prop:percwin}}\ \
We start off with the Matrix Perceptron bound.
For brevity, we write $X_t$ instead of $X^{\mbox{\scriptsize p}}_t$. Also, let
$\langle A,B\rangle$ be a shorthand for the inner product $\tr(A^TB)$. 
%Note that $\langle ~,~\hrangle$ is an inner product.
We can write
\begin{align}
\langle X_{t},X_{t}\rangle 
&=\tr((\Psi^+)^T(\be_{i_{t}}-\be_{j_{t}})(\be_{i_t}-\be_{j_t})^T\Psi^+
(\Psi^+)^T(\be_{i_{t}}-\be_{j_{t}})(\be_{i_{t}}-\be_{j_{t}})^T\Psi^+)\notag\\
&=\tr((\be_{i_{t}}-\be_{j_{t}})^T\Psi^+(\Psi^+)^T(\be_{i_t}-\be_{j_t})
(\be_{i_t}-\be_{j_t})^T\Psi^+(\Psi^+)^T(\be_{i_{t}}-\be_{j_{t}}))\notag\\
&=((\be_{i_{t}}-\be_{j_{t}})^T L^+(\be_{i_{t}}-\be_{j_t}))^2\notag\\
&= (R^G_{i_t,j_t})^2\notag\\
&\leq R^2~.\notag
\end{align}
Moreover, for any $k \in \{1,\ldots,K\}$, define $K$ vectors $\bu_1, \ldots, \bu_K \in \R^n$
as follows.
\[
\bu_k = (u_{k,1}, \ldots, u_{k,n})^\top,\qquad {\mbox{with}}\qquad u_{k,i} = [k=y_i]~,
\]
being $y_i$ the label of the $i$-th vertex of $G$. Now,
if we let $U:=\Psi\left( \sum_{k=1}^K \bu_k\bu_k^\top\right) \Psi^\top$, we have 
\begin{align*}
\langle U, X_t\rangle&=\tr(U^TX_t)\\
&=\sum_{k=1}^K\tr(\Psi \bu_k\bu_k^T\Psi^T(\Psi^+)^T(\be_{i_t}-\be_{j_t})(\be_{i_t}-\be_{j_t})^T\Psi^+)\\
&=\sum_{k=1}^K\tr((\be_{i_t}-\be_{j_t})^T\Psi^+\Psi \bu_k\bu_k^T\Psi^T(\Psi^+)^T(\be_{i_t}-\be_{j_t}))\\
&=\sum_{k=1}^K((\be_{i_t}-\be_{j_t})^T\Psi^+\Psi \bu_k)^2~.
\end{align*}
By definition of pseudoinverse, 
$\Psi(\Psi^+\Psi \bu_k)=(\Psi\Psi^+\Psi)\bu_k=\Psi \bu_k$ 
for all $k = 1,\ldots,K$. 
Hence (recall Section \ref{sec:spg}), 
$\Psi^+\Psi \bu_k=\bu_k + c\one$ for some $c\in\mathbb{R}$. 
We therefore have that $(\be_{i_t}-\be_{j_t})^T\Psi^+\Psi \bu_k = u_{k,{i_t}}- u_{k,j_t}$,
i.e.,
\[
\langle U, X_t\rangle = \sum_{k=1}^K(u_{k,i_t}- u_{k,j_t})^2~.
\]
Now, if $y_{i_t,j_t}=0$ (i.e., $y_{i_t}= y_{j_t}$) 
then for all $k$ we have $u_{k,i_t}-u_{k,j_t}=0$, so that $\langle U, X_t\rangle = 0$.
On the other hand, if $y_{i_t,j_t}=1$ (i.e., $y_{i_t} \neq y_{j_t}$) 
then there exist distinct $a,b \in\{1,\ldots, K\}$ such that 
$|u_{a,i_t}-u_{a,j_t}|= |u_{b,i_t}-u_{b,j_t}|=1$, and for all other $k\neq a,b$ 
we have $u_{k,i_t}-u_{k,j_t}=0$. So, in this case  $\langle U, X_t\rangle = 2$.

This gives the linear separability condition of sequence $(X_1,y_{i_1,j_1}), (X_2,y_{i_2,j_2}), \ldots$ w.r.t. $U$.

Finally, we bound $\langle U, U\rangle$. 
Let $\Phi^G_{a,b}:=\{(i,j)\in E\,:\,y_i=a,\,y_j=b \}$. We have:
\begin{align*}
\langle U, U\rangle&=\tr(U^TU)\\
&=\tr\left(\left(\sum_{a=1}^K\Psi\bu_a\bu_a^T\Psi^T\right)\left(\sum_{b=1}^K\Psi\bu_b\bu_b^T\Psi^T\right)\right)\\
&=\sum_{a=1}^K\sum_{b=1}^K\tr(\Psi\bu_a\bu_a^T\Psi^T\Psi\bu_b\bu_b^T\Psi^T)\\
&=\sum_{a=1}^K\sum_{b=1}^K\tr(\bu_b^T\Psi^T\Psi\bu_a\bu_a^T\Psi^T\Psi\bu_a) \\
&=\sum_{a=1}^K\sum_{b=1}^K(\bu_b^T\Psi^T\Psi\bu_a)^2\\
&=\sum_{a=1}^K\left( |\Phi^G_a|^2 + \sum_{b\neq a}|\Phi^G_{a,b}|^2\right)~.
\end{align*}
So, noticing that $\sum_{b\,:\,b\neq a}|\Phi^G_{a,b}|=|\Phi^G_a|$ 
and hence that $\sum_{b\,:\,b\neq a}|\Phi^G_{a,b}|^2\leq|\Phi^G_a|^2$, we conclude that
\[
\langle U, U\rangle \leq 2|\Phi^G|^2~.
\]
With the above handy, the mistake bound on $M^{\mbox{\sc p}}$ easily follows from the standard
analysis of the Perceptron algorithm with nonzero threshold.

By a similar token, the bound on $M^{\mbox{\sc w}}$ follows from the arguments in \cite{w07},
after defining $U$ to be a normalized version of the one we defined above for Matrix Perceptron,
and noticing that $X^{\mbox{\sc w}}_t$ in Algorithm \ref{a:2} are positive semidefinite and 
normalized to trace 1.

% !TEX root = main7.tex

\subsection{Missing proofs from Section \ref{s:efficient}}\label{ssa:missingeff}
The following lemma relies on the equivalence between effective resistance $\er^G_{i,j}$ of an
edge $(i,j)$ and its probability of being included in a randomly drawn spanning tree.
\begin{lemma}\label{l:expected}
Let $(G,\by)$ be a labeled graph, and $T$ be a spanning tree of $G$ 
drawn uniformly at random.
Then, for all $k = 1, \ldots, K$, we have:
\begin{enumerate} 
\item $\E[|\Phi^T_k|] =  \sum_{(i,j)\in\Phi^G_k} \er^G_{i,j}$, and
\item $\E[|\Phi^T_k|^2]\leq 2(\sum_{(i,j)\in\Phi^G_k}\er^G_{i,j})^2$~.
\end{enumerate}
\end{lemma}
\begin{proof}
%Proof of Lemma~\ref{l:expected} 
Set $s = |\Phi^G_k|$ and $\Phi^G_k=\{(i_1, j_1),(i_2, j_2), \ldots, (i_{s}, j_{s})\}$. 
Also, for $\ell = 1, \ldots, s$, let $X_{\ell}$ be the random variable which is $1$ if 
$(i_{\ell},j_{\ell})$ is an edge of $T$, and $0$ otherwise.
From $\E[X_{\ell}]=\er^G_{i_{\ell},j_{\ell}}$ we immediately have 1). In order to 
prove 2), we rely on the negative correlation of variables $X_{\ell}$, i.e.,
that $\E[X_{\ell}\,X_{\ell'}] \leq \E[X_{\ell}]\,\E[X_{\ell'}]$ 
for $\ell\neq \ell'$ (see, e.g., \cite{lp10}). Then we can write
\begin{align*}
\mathbb{E}(|\Phi^T_k|^2)
&=\mathbb{E}\left[\left(\sum_{\ell=1}^{s} X_{\ell}\right)^2\right]\\
&=\mathbb{E}\left[\sum_{\ell=1}^{s}\sum_{\ell'=1}^{s}X_{\ell} X_{\ell'}\right]\\
%&=\mathbb{E}\left[\sum_{\ell=1}^{s} X_{\ell} +\sum_{\ell=1}^{s}\sum_{\ell'\neq \ell}X_{\ell} X_{\ell'}\right]\\
&=\sum_{\ell=1}^{s}\mathbb{E}[X_{\ell}] +\sum_{\ell=1}^{s}\sum_{\ell'\neq \ell}\mathbb{E}[X_{\ell} X_{\ell'}]\\
%&=\sum_{\ell=1}^{s}\mathbb{P}(X_{\ell}=1) +\sum_{\ell=1}^{s}\sum_{\ell'\neq \ell}\mathbb{P}(X_{\ell}=1,\,X_{\ell'}=1)\\
&\leq\sum_{\ell=1}^{s}\E[X_{\ell}] +\sum_{\ell=1}^{s}\sum_{\ell'\neq \ell}\E[X_{\ell}]\,\E[X_{\ell'}]~.
\end{align*}
Now, for any spanning tree $T$ of $G$,  
if $s \geq 1$ then it must be the case that $|\Phi^T_k|\geq 1$, and hence
\(
\sum_{\ell=1}^{s}\E[X_{\ell}] =\E[|\Phi^T_k|] \geq 1~.\ 
\)
Combined with the above we obtain:
%
% \begin{align}
% \E[|\Phi^T_k|^2] 
% &\leq\sum_{p=1}^{|\Phi_a|}\mathbb{P}(X_p=1) +\sum_{p=1}^{|\Phi_a|}\sum_{q\neq p}\mathbb{P}(X_p=1)\mathbb{P}(X_q=1)\\
% &\leq\left(\sum_{q=1}^{|\Phi_a|}\mathbb{P}(X_q=1)\right)\sum_{p=1}^{|\Phi_a|}\mathbb{P}(X_p=1) +\sum_{p=1}^{|\Phi_a|}
% \sum_{q\neq p}\mathbb{P}(X_p=1)\mathbb{P}(X_q=1)\\
% &\leq 2\left(\sum_{p=1}^{|\Phi_a|}\mathbb{P}(X_p=1)\right)\left(\sum_{q=1}^{|\Phi_a|}\mathbb{P}(X_q=1)\right)\\
% &=2\left(\sum_{p=1}^{|\Phi_a|}\mathbb{P}(X_p=1)\right)^2\\
% &=2\left(\sum_{p=1}^{|\Phi_a|}r^G_{i^a_pj^a_p}\right)^2
% \end{align}
%
\[
\E[|\Phi^T_k|^2] 
\leq \left(\sum_{\ell=1}^{s}\E[X_{\ell}]\right)^2 +\sum_{\ell=1}^{s}\sum_{\ell'\neq \ell}\E[X_{\ell}]\,\E[X_{\ell'}]
\leq 2\,\left(\sum_{\ell=1}^{s}\E[X_{\ell}]\right)^2
=2\left(\sum_{\ell=1}^{s} \er^G_{i_{\ell},j_{\ell}}\right)^2~,
\]
as claimed.
\end{proof}

\ \\
\noindent{\bf Proof of Theorem \ref{t:winnowandperceptron} }
From Proposition~\ref{prop:percwin} we have that if we execute Matrix Winnow on $B=B_T$ with 
similarity instances constructed from $\Psi_B$, then the number $M$ of mistakes satisfies
\[
M = \cO\left(|\Phi^{B}|\,D_{B}\,\log n \right)~,
\]
where $D_{B}$ is the resistance diameter of $B$. 
Since $B$ is a tree, its resistance diameter is equal to its diameter, 
which is $\cO(\log n)$. Moreover, $|\Phi^{B}_k| = \cO(|\Phi^{T}_k| \log n)$, for $k = 1,\ldots, K$,
hence $|\Phi^{B}| = \cO(|\Phi^{T}| \log n)$.
Plugging back, taking expectation over $T$, and using Lemma \ref{l:expected}, 1)
proves the Matrix Winnow bound. 
Similarly, if we run the Matrix Perceptron algorithm on $B$ with 
similarity instances constructed from $\Psi_B$ then
\[
M = \cO\left(|\Phi^{B}|^2\,D^2_{B} \right)~.
\]
Proceeding as before, in combination with Lemma \ref{l:expected}, 2), proves the Matrix Perceptron bound.

\ \\
\noindent{\bf Proof of Theorem \ref{t:equivalence} }
First of all, the fact that the algorithm is $\cO(\log^2 n)$ per round easily follows from the fact
that, since $B$ is a balanced binary tree, the sizes of sets $\cP_t$ (prediction step in (\ref{e:predictionstep}))
and  $\cS_t$ (update step in (\ref{e:updatestep})) are both $\cO(\log n)$.

As for initialization time, a naive implementation would require $\cO(n^2)$ 
(we must build the zero matrix $F$). 
We now outline a method of growing a data structure that stores a representation of $F$ online 
for which the initialisation time is only $\cO(n)$, while keeping the per round
time to $\cO(\log^2 n)$.
For every vertex $\ell$ in $B$ the algorithm maintains a subtree $B_{\ell}$ of $B$, 
initially set to $\{\rho\}$, being $\rho$ the root of $B$.
At every vertex $\ell'\in B_{\ell}$ is stored the value $F_{\ell,\ell'}$.
At the start of time $t$, the algorithm climbs $B$ from $i_t$ to $\rho$, 
in doing so storing the ordered list $\cL_{i_t}$ of vertices in the path from $\rho$ to 
$i_t$. The same is done with $j_t$. The set $\cS_t$ is then computed. 
For all $\ell\in\cS_t$, the tree $B_{\ell}$ is then extended to include the vertices in 
$\cN_t$ and the path from $i_t$ (note that for each $\ell\in\cS_t$ this takes only 
$\cO(\log n)$ time, since we have the list $\cL_{i_t}$). 
Whenever a new vertex $\ell'$ is added to $B_{\ell}$, the value $F_{\ell,\ell'}$ is set to zero.
Hence, we initialize $F$ ``on demand", the only initialization step being the allocation
of the BST, i.e., $\cO(n)$ time.

We now continue by showing the equivalence of the sequence of predictions
issued by (\ref{e:predictionstep}) 
%and (\ref{e:updatestep})
to those of the Matrix Perceptron algorithm with similarity instances constructed from $\Psi_B$.

For every $\ell\in\cS_t$ define $\Lambda_t(\ell)$ as the maximal subtree of $B$ that contains 
$\ell$ and does not contain any nodes in $\cP_t\setminus\{i_t,j_t\}$. 
\begin{lemma}\label{lamfacts}
$\Lambda_t(\cdot)$ defined above enjoys the following properties (see Figure \ref{f:lambda}, left, for reference).
\begin{enumerate}
\item For all $\ell$, $\Lambda_t(\ell)$ is uniquely defined;
\item Any subtree $T$ of $B$ that has no vertices from $\cP_t\setminus\{i_t,j_t\}$ 
(and hence any of the trees $\Lambda_t$) contains at most one vertex from $\cS_t$;
\item The subtrees $\{\Lambda_t(\ell)\,:\,\ell\in\cS_t\}$ are pairwise disjoint;
\item The set $\{\Lambda_t(\ell)\,:\,\ell\in\cS_t\}\cup(\cP_t\setminus\{i_t,j_t\})$ covers $B$ 
(so in particular $\{\Lambda_t(\ell)\,:\,\ell\in\cS_t\}$ covers the set of leaves of $B$).
\end{enumerate}
\end{lemma}
\begin{proof}
\begin{enumerate}
\item Suppose we have subtrees $T$ and $T'$ with $T\neq T'$ that both satisfy the conditions of 
$\Lambda_t(\ell)$. Then w.l.o.g assume there exists a vertex $\ell'$ in $T$ that is not in $T'$. 
Since $T$ and $T'$ are both connected and both contain $\ell$, the subgraph $T\cup T'$ of $B$ 
is connected and is hence a subtree. Since neither $T$ nor $T'$ contains vertices in 
$\cP_t\setminus\{i_t,j_t\}$, $T\cup T'$ does not contain any such either. 
Hence, because $T'$ is a strict subtree of $T\cup T'$, we have contradicted the maximality of $T'$.
\item Suppose $T$ has distinct vertices $\ell, \ell'\in\cS_t$. 
Since $T$ is connected, it must contain the path in $B$ from $\ell$ to $\ell'$. 
This path goes from $\ell$ to the neighbor of $\ell$ that is in $\cP_t\setminus\{i_t,j_t\}$, 
then follows the path $\cP_t\setminus\{i_t,j_t\}$ (in the right direction) until a neighbor of $\ell'$ 
is reached. The path then terminates at $\ell'$. Such a path contains at least one vertex in 
$\cP_t\setminus\{i_t,j_t\}$, contradicting the initial assumption about $T$.
\item Assume the converse -- that there exist distinct $\ell, \ell'$ in $\cS_t$ 
such that $\Lambda_t(\ell)$ and $\Lambda_t(\ell')$ share vertices. 
Then, since $\Lambda_t(\ell)$ and $\Lambda_t(\ell')$ are connected, 
$\Lambda_t(\ell)\cup\Lambda_t(\ell')$ must also be connected (and hence must be a subtree of $B$). 
Since $\Lambda_t(\ell)\cup\Lambda_t(\ell')$ shares no vertices with $\cP_t\setminus\{i_t,j_t\}$, 
and contains both $\ell$ and $\ell'$ (which are both in $\cS_t$), the statement in Item 2 above
is contradicted.
\item Assume that we have a $\ell\in B\setminus(\cP_t\setminus\{i_t,j_t\})$. 
Then let $P'$ be the path from $\ell$ to the (first vertex encountered in) the path $\cP_t\setminus\{i_t,j_t\}$. 
Let $\ell'$ be the second from last vertex in $P'$. 
Then $\ell'$ is a neighbor of a vertex in $\cP_t$, but is not in $\cP_t\setminus\{i_t,j_t\}$, so it must be 
in $\cS_t$. This implies that the path $P''$ that goes from $\ell$ to $\ell'$ 
contains no vertices in $\cP_t\setminus\{i_t,j_t\}$ and is therefore (Item 1) 
a subtree of $\Lambda_t(\ell')$. Hence, $\ell\in\Lambda_t(\ell')$.
\end{enumerate}
\end{proof}
\begin{figure}[h!]
\centering
\includegraphics[width=0.9\textwidth]{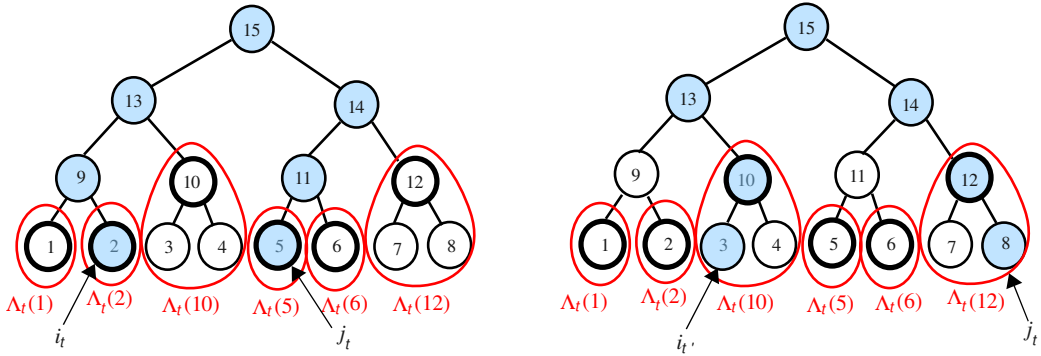}
\vspace{-5.3in}
\caption{ 
{\bf Left: } The same BST as in Figure \ref{f:table} with $i_t =2$ and $j_t = 5$. 
Light blue vertices are those in $\cP_t$. 
Thick-bordered vertices are those in $\cS_t$. 
Since vertices $3$ and $4$ are in $\Lambda_t(10)$, we have $\tilde{f}_t(3) = \tilde{f}_t(4)=f_t(10)$. 
Since vertices $7$ and $8$ are in $\Lambda_t(12)$, we have $\tilde{f}_t(7) = \tilde{f}_t(8)=f_t(12)$. 
For all other vertices $\ell$, we have $\tilde{f}_t(\ell)=f_t(\ell)$. 
{\bf Right: } The same BST as in Figure \ref{f:table} with path $\cP_{t'}$ (light blue vertices)
having endpoints $i_{t'} = 3$ and $j_{t'} = 8$. Thick-bordered vertices are still those in $\cS_t$. 
Path $\cP_{t'}$ intersects $\cS_t$ at two vertices, $10$ and $12$, which means that 
$3\in\Lambda_t(10)$ and $8\in\Lambda_t(12)$. We have $\tilde{f}_t(3)=f_t(10)$, 
$\tilde{f}_t(8)=f_t(12)$, and $((\be_3-\be_8)L^+(\be_2-\be_5))^2
=(\tilde{f}_t(3)-\tilde{f}_t(8))^2 =(f_t(10)-f_t(12))^2$.
\label{f:lambda}
}
\vspace{-0.1in}
\end{figure}

\begin{lemma}\label{l:lambda}
Let $L$ be the Laplacian matrix of $B$, and $\ell, \ell' \in \cS_t$. Then
for any pair of vertices $\kappa$ and $\kappa'$ of $B$ with $\kappa\in\Lambda_t(\ell)$ 
and $\kappa'\in\Lambda_t(\ell')$ we have
%
%\begin{equation}
\[
(\be_{\kappa}-\be_{\kappa'})^T L^+(\be_{i_t}-\be_{j_t})= f_t(\ell')- f_t(\ell)~,
\]
%\end{equation}
where $\be_i$ is the $i$-th element in the canonical basis of $\R^{2n-1}$.
\end{lemma}
\begin{proof}
We first extend the tagging function $f_t$ to all vertices of $B$ via 
the vector\footnote
{
In our notation, we interchangeably view $\tilde{f}$ both as a tagging 
function from the $2n-1$ vertices of $B$ to the natural numbers and as a $(2n-1)$-dimensional vector.
} 
$\tilde{f}_t$ as follows (note that, by Lemma \ref{lamfacts}, $\tilde{f}_t$ is well defined):
\begin{enumerate}
\item For all $\ell\in\cP_t\setminus\{i_t,j_t\}$, set $\tilde{f}_t(\ell) =f_t(\ell)$;
\item For all $\ell'\in\cS_t$ and $\ell\in\Lambda_t(\ell')$, set $\tilde{f}_t(\ell) = f_t(\ell')$.
\end{enumerate}
\begin{claim}\label{c:1}
$L\tilde{f}_t = \be_{j_t}-\be_{i_t}$.
\end{claim}
\noindent{\bf Proof of claim.} For any vertex $\kappa$ of $B\setminus\{i_t,j_t\}$ one of the following holds:
\begin{enumerate}
\item If $\kappa\in\cP_t$, then $\kappa$ has a neighbor $\kappa_1$ with 
$\tilde{f}_t(\kappa_1) = \tilde{f}_t(\kappa)-1$, one neighbour $\kappa_2$ with 
$\tilde{f}_t(\kappa_2)=\tilde{f}_t(\kappa)+1$, and (unless $\kappa$ is the root of $B$)
one neighbour $\kappa_3$ with $\tilde{f}_t(\kappa_3)= \tilde{f}_t(\kappa)$. 
We therefore have that 
$[L\tilde{f}_t]_{\kappa} = 3\tilde{f}_t(\kappa) 
- \tilde{f}_t(\kappa_1)-\tilde{f}_t(\kappa_2)- \tilde{f}_t(\kappa_3) = 0$.
\item If $\kappa\in\cN_t\setminus\cP_t$, then $\kappa$ has one neighbor $\kappa_1$ in $\cP_t$ and we have 
$\tilde{f}_t(\kappa_1)= \tilde{f}_t(\kappa)$. 
Let $T_{\kappa}$ be the subtree of $B$ containing exactly vertex $\kappa$ and all 
neighbors of $\kappa$ bar $\kappa_1$. 
Since $\cP_t$ is connected, it contains $\kappa_1$ and does not contain $\kappa$, 
none of the other neighbors of $\kappa$ being in $\cP_t$. 
Hence $T_{\kappa}$ is a subtree of $B$ that contains $\kappa$ and no vertices from 
$\cP_t\setminus\{i_t,j_t\}$, and so by Lemma \ref{lamfacts}, item $1$ it must be a subtree of $\Lambda_t(\kappa)$. 
Hence, by definition of $\tilde{f}_t$, all vertices $\kappa_2$ in $T_{\kappa}$ satisfy 
$\tilde{f}_t(\kappa_2)= \tilde{f}_t(\kappa)$. This implies that for all neighbors $\kappa_3$ of $\kappa$
we have $\tilde{f}_t(\kappa_3)= \tilde{f}_t(\kappa)$, which in turn gives $[L\tilde{f}_t]_k = 0$.
\item If $\kappa\notin\cN_t$ then, by Lemma \ref{lamfacts} item $4$, let $\kappa$ be contained in 
$\Lambda_t(\ell)$ for some $\ell\in\cS_t$. Let $T_{\kappa}$ be the subtree of $B$ containing exactly 
vertex $\kappa$ and all neighbors of $\kappa$. Note that $T_{\kappa}$ 
is a subtree of $B$ that contains $\kappa$ and no vertices from $\cP_t\setminus\{i_t,j_t\}$. 
Since $\Lambda_t(\ell)$ also contains $\kappa$ (hence $\Lambda_t(\ell)\cup T_{\kappa}$ is connected), 
we have that $\Lambda_t(\ell)\cup T_{\kappa}$ is a subtree of $B$ that contains 
$\ell$ and no vertices from $\cP\setminus\{i_t,j_t\}$. By Lemma \ref{lamfacts} item $1$, this implies that
$\Lambda_t(\ell)\cup T_{\kappa}$ is a subtree of (and hence equal to) $\Lambda_t(\ell)$. 
Hence, by definition of $\tilde{f}_t$, we have that $\tilde{f}_t$ is identical on $T_{\kappa}$. 
Thus all neighbors $\kappa_1$ of $\kappa$ satisfy $\tilde{f}_t(\kappa_1)= \tilde{f}_t(\kappa)$, implying again 
$[L\tilde{f}_t]_{\kappa} = 0$.
\end{enumerate}
So in either case $[L\tilde{f}_t]_{\kappa} = 0$.

Finally, let $i'_t$ be the neighbor of $i_t$ in $B$. 
We have $[L\tilde{f}_t]_{i_t} = \tilde{f}_t(i_t) - \tilde{f}_t(i'_t) = 1-2=-1$. Similarly, we have
$[L\tilde{f}_t]_{j_t} = 1$. Putting together, we have shown that $L\tilde{f}_t=\be_{j_t}-\be_{i_t}$,
thereby concluding the proof of Claim \ref{c:1}.

Now, by definition of pseudoinverse,
\[
L\tilde{f}_t =LL^+L\tilde{f}_t =LL^+(\be_{j_t}-\be_{i_t})~.
\]
This mplies that $L(\tilde{f}_t-L^+(\be_{j_t}-\be_{i_t}))=0$. Therefore (see Section \ref{sec:spg})
there exists a constant $c$ such that $\tilde{f}_t = L^+(\be_{j_t}-\be_{i_t})+c\one$. 
%This means that  $\tilde{f}_t=L^+(e_{j_t}-e_{i_t})+c1$ for some $c\in\mathbb{R}$.
From the definition of $\tilde{f}$ we can write
\begin{align*}
{f}_t(\ell')-{f}_t(\ell) 
& = \tilde{f}_t(\kappa') - \tilde{f}_t(\kappa)\\
& =([L^+(\be_{j_t}-\be_{i_t})]_{\kappa'} -c) -([L^+(\be_{j_t}-\be_{i_t})]_{\kappa}-c)\\
& =(\be_{\kappa}-\be_{\kappa'})^T L^+(\be_{i_t}-\be_{j_t}),
\end{align*}
as claimed.
\end{proof}
\begin{lemma}\label{l:twocase}
Let $L$ be the Laplacian matrix of $B$, and $\kappa, \kappa'$ be two vertices of $B$. 
Let $\cP$ be the path from $\kappa$ to $\kappa'$ in $B$. 
Then for any $t$ either $|\cP\cap\cS_t|\leq 1$ or $\cP\cap\cS_t = \{\ell,\ell'\}$,
for two distinct vertices $\ell$ and $\ell'$. No other cases are possible. Moreover,
\[
((\be_{\kappa}-\be_{\kappa'})^T L^+(\be_{i_t}-\be_{j_t}))^2
=
\begin{cases}
0                   & {\mbox{if $|\cP\cap\cS_t|\leq 1$}}\\
(f_t(\ell)-f_t(\ell'))^2   & {\mbox{if $\cP\cap\cS_t = \{\ell,\ell'\}$}}~.
\end{cases}
\]
\end{lemma}
\begin{proof}
By Lemma \ref{lamfacts} item $4$, we have two possible cases only:
\begin{enumerate}
\item There exists $\ell\in\cS_t$ such that  both $\kappa$ and $\kappa'$ are in $\Lambda_t(\ell)$: 
In this case (since $\Lambda_t(\ell)$ is connected) the path $\cP$ lies in $\Lambda_t(\ell)$. 
Since, by Lemma \ref{lamfacts} item $2$, no $\ell'\in\cS_t$ with $\ell'\neq \ell$ 
can be in $\Lambda_t(\ell)$, it is only ever possible that $\cP$ contains at most one vertex 
$\ell$ (if any) of $\cS_t$.
\item There exist two distinct nodes $\ell, \ell'\in\cS_t$ such that $\kappa\in\Lambda(\ell)$ 
and $\kappa'\in\Lambda(\ell')$. In this case, $\cP$ corresponds to the following path: 
First go from $\kappa$ to $\ell$ (by Lemma \ref{lamfacts} item $2$, since this path lies 
in $\Lambda(\ell)$ the only vertex in $\cS_t$ that lies in the section of the path is $\ell$); 
then go to the neighbor of $\ell$ that is in $\cP_t\setminus\{i_t, j_t\}$; then follow the path 
$\cP_t\setminus\{i_t, j_t\}$ until you reach the neighbor of $\ell'$ (this section of $\cP$ 
contains no vertices in $\cS_t$); then go from $\ell'$ to $\kappa$ 
(by Lemma \ref{lamfacts} item $2$, since this path lies in $\Lambda(\ell')$ 
the only vertex in $\cS_t$ that lies in this section of the path is $\ell'$). 
Thus, $\cP\cap\cS_t=\{\ell,\ell'\}$.
\end{enumerate}
The result then follows by applying Lemma \ref{l:lambda} to the two cases above.
\end{proof}
Figure \ref{f:lambda} illustrates the above lemmas by means of an example.

To conclude the proof, let $\langle A,B\rangle$ be a shorthand for $\tr(A^\top B)$.
We see that from Algorithm \ref{a:2}, Lemma \ref{l:twocase}, and the definition of $F$ in 
(\ref{e:predictionstep}) we can write
\begin{align*}
\langle W_{t},X_{t}\rangle 
&= \sum_{t'=1, t' \in \cM}^{t-1} (2y_{i_{t'},j_{t'}}-1) \langle X_{t'},X_{t}\rangle\\
&= \sum_{t'=1, t' \in \cM}^{t-1} (2y_{i_{t'},j_{t'}}-1) ((\be_{i_{t'}}-\be_{j_{t'}})^T L^+(\be_{i_t}-\be_{j_t}))^2\\
&= \sum_{(\ell,\ell')\in\cP_{t}^2}F_{\ell,\ell'}~,
\end{align*}
where $\cM$ is the set of mistaken rounds, and the second-last equality
follows from a similar argument as the one contained in the proof of Proposition
\ref{prop:percwin}.
Threshold $2\log n$ in (\ref{e:predictionstep}) is an upper bound on the radius squared
$\langle X_t,X_t \rangle$ of instance matrices (denoted by $R^2$ in Algorithm \ref{a:2}). 
In fact, from the proof of Proposition \ref{prop:percwin},
\[
\max_t \langle X_t,X_t \rangle \leq \max_{(i,j) \in V^2} ((\be_i-\be_j)^\top L^+ (\be_i-\be_j))^2 
= \max_{(i,j) \in V^2} (R^G_{i,j})^2~,
\]
which is upper bounded by the square of the diameter $2\log n$ of $B$.

% !TEX root = main7.tex 

\subsection{Ancillary results, and missing proofs from Section \ref{s:unknown}}\label{ssa:unknown}
This section contains the proof of Theorem \ref{t:unknown}, along with preparatory results.

\subsection*{A digression on cuts and directed paths}
Given\footnote
{
The reader familiar with the theory of matroids will recognize what is recalled here
as a well known example of a regular matroid on graphs.
% along with the associated
% dual matroid. 
One can learn about them in standard textbooks/handbooks, e.g., \cite[Ch.6]{gy03}.
} 
a connected and unweighted graph $G = (V,E)$, with $n = |V|$ vertices and $m = |E|$ edges,
any partition of $V$ into two subsets induces a {\em cut} over $E$. 
%As we recalled in Section \ref{s:...}, 
A cut is a {\em cutset} if it is induced by a two-connected component partition of $V$.
Fix now any spanning tree $T$ of $G$ (this will be the one constructed by the algorithm
at the end of the game, based on the paths produced by the adversary -- see 
Section \ref{ss:unknownaa}). 
In this context, the $n-1$ edges of $T$ are often called {\em branches}
and the remaining $m-n+1$ edges are often called {\em chords}.
Any branch of $T$ cuts the tree into two components (it is therefore a cutset), 
and induces a two-connected component partition over $V$. 
Any such cutset is called a {\em fundamental cutset} of $G$ (w.r.t. $T$). 
% Similarly, any chord $(i,j)$ determines a circuit\footnote
% {
% Generally speaking, a circuit is a closed path possibly with repeated nodes, 
% but no repeated edges.
% }
% in $G$ made up of edge $(i,j)$ and the unique path in $T$ that connects $i$ to $j$. 
% Such a circuit is called a {\em fundamental circuit} of $G$ (w.r.t. $T$). 
Cuts 
% and circuits 
are 
always subsets of $E$, hence they can naturally be represented as (binary) 
indicator vectors with $m$ 
components. For reasons that will be clear momentarily, it is also convenient 
to assign each edge an 
orientation (tail vertex to head vertex) 
% each circuit an orientation (clockwise or counterclockwise)
and each cut an inward/outward direction. 
In particular, it is customary to give 
% a fundamental circuit the orientation of its chord, and 
a fundamental cut the orientation of its branch.
As a consequence of orientations/directions, cuts 
% and circuits 
are rather represented as $m$-dimensional 
vectors whose components have values in $\{-1,0,1\}$.
%
%\vspace{1in}
%

\begin{figure}[t!h!]
\begin{picture}(20,290)(20,290)
\scalebox{0.75}{\includegraphics{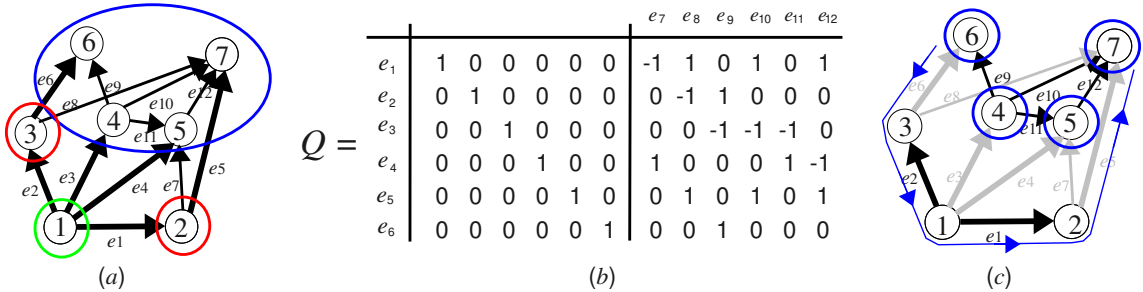}}
\end{picture}
\vspace{-2.15in}
\caption{{\bf (a)} A graph with oriented edges, tagged $e_1$ through $e_{12}$. 
A spanning tree $T$ is denoted by the thick edges. 
The cutset $\{4,5,6,7\}$ is depicted in blue. This cut has inward branches 
$e_3$, $e_4$, $e_5$, and $e_6$, and no outward branches. 
The red cutset separating vertex $3$ from the rest has inward branch $e_2$, and outward
branch $e_6$. Two other (nodal) cutsets are shown: one separating vertex
$2$ from the rest, and the other separating vertex $1$ from the rest.
{\bf (b)} The fundamental cutset matrix
% and fundamental circuit ({\bf bottom}) matrices, 
% denoted by $Q$ 
% and $C$, respectively, 
associated with the chosen spanning tree.
Viewed as a 12-dimensional vector, 
% 
% the red cut is 
% $q_{\{3,4,5\}} = [0,-1,-1,-1,0,1,-1,1,1,1,0,1]$, 
% and can be expressed as linear combination of rows $Q_i$ of cut matrix $Q$ as $-Q_2-Q_3-Q_4+Q_6$, i.e., by the
% vector of coefficients $\bu_{\{3,4,5\}} = [0,-1,-1,-1,0,1]^\top$.
% Notice that chord $e_7$ is inward (coefficient $-1$ in $q_{\{3,4,5\}}$), 
% while chords $e_8$, $e_9$, $e_{10}$, and $e_{12}$ are outward
% (coefficients $+1$ in $q_{\{3,4,5\}}$). Similarly, 
% 
the blue cut is $q_{\{4,5,6,7\}} =  (0,0,-1,-1,-1,-1,-1,-1,0,0,0,0)$, and can be represented 
as linear combination of rows $Q_i$ of 
%cut matrix 
$Q$ as $-Q_3-Q_4-Q_5-Q_6$, i.e.,
by the vector of coefficients $\bu_{\{4,5,6,7\}} = (0,0,-1,-1,-1,-1)^\top$. 
Notice that chords $e_7$ and $e_8$ are both inward (coefficients $-1$ in $q_{\{4,5,6,7\}}$).
%
% Obverve that the (noncutset) 
% cut coefficients $\bu_{\{3,4,5\}}$ can always be represented as $\bu_{\{3\}} + \bu_{\{4,5\}}$, 
% where $\bu_{\{3\}} = [0,-1,0,0,0,1]^\top$
% and $\bu_{\{4,5\}} = [0,0,-1,-1,0,0]^\top$ are the vector of coefficients corresponding 
% to the cutsets $\{3\}$ and 
% $\{4,5\}$, respectively.
{\bf (c)} 
% A similar thing happens for nodes $4$ and $5$ in the red cut on the top left.
The connectivity structure induced by the selected spanning tree on the blue cutset in (a).
For ease of reference, all edges in the blue cut have been turned into light gray.
Vertices $4$, $5$, $6$, and $7$ are all connected in $G$ under the blue cut, but are they
are {\em all disconnected} in $T$.
The path $6 \rightarrow 3 \rightarrow 1 \rightarrow 2 \rightarrow 7$ (depicted in blue)
connects in $T$ vertex $6$ to vertex $7$, and
is represented by path $\bp = (1,-1,0,0,1,-1,0,0,0,0,0,0)^\top$, hence $Q\bp = (1,-1,0,0,1,-1)^\top$.
This path departs from the blue (disconnected) cluster $\{4,5,6,7\}$ through edge $e_6$ (traversed ``the wrong way") 
and returns to this set via $e_5$. 
Notice that $\bu_{\{4,5,6,7\}}^\top Q\bp = 0$.
%\vspace{-0.3in}
}
\label{f:4a}
\end{figure}

Figure \ref{f:4a} (a) gives an example. 
In this figure, 
% the branches are the edges $e_1$ through
% $e_6$, and the chords are the edges $e_7$ through $e_{12}$. We will follow this convention throughout:
% low index edges are the branches and high index ones are the chords.
% Moreover, 
all edges are directed from the low
index vertex to the high index vertex. 
% In Figure \ref{f:1} (top right), 
Branch $e_1$ determines a cutset (more precisely, a 
fundamental cutset w.r.t. to the depicted spanning tree) separating vertices 2 and 7 from the remaining ones. 
Edge $e_6$ isolates just vertex 6 from the rest (again, a fundamental cutset). 
The fundamental cut determined by branch $e_1$ is represented by vector 
$q = (1, 0, 0, 0, 0, 0, -1, 1, 0, 1, 0, 1)$.
This is because if we interpret the orientation of branch $e_1$ as outward to the cut, then $e_7$ is inward,
$e_8$ is outward, as well as $e_{10}$ and $e_{12}$.
% Chord $e_8$ determines the fundamental circuit $1 - 3 - 7 - 2 - 1$, involving edges 
% $e_2$, $e_8$, $e_5$, and $e_1$. Since it has to be taken clockwise (according to chord $e_8$), 
% this circuit is represented by the vector $c = [-1, 1, 0, 0, -1, 0, 0, 1, 0, 0, 0, 0]$. 
The matrix in Figure \ref{f:4a} (b) contains as rows all fundamental
cutsets. This is usually called the {\em fundamental cutset matrix}, often denoted by $Q$ (recall that
this matrix depends on spanning tree $T$ -- for readability, we drop this dependence from our notation). 
Matrix $Q$ has rank $n-1$. 
% One can show (and easily inspect here) that the last $m-n+1 = 6$ 
% colums contain the fundamental circuits involving the chords on top of each column. For instance, 
% the nonzero
% elements in column tagged by $e_7$ correspond to the branches of the fundamental circuit determined 
% by $e_7$ itself. 
% Notice, however, that the orientation is reversed. This holds in general: Dual to the fundamental
% cutset matrix $Q$ is the {\em fundamental circuit matrix} $C$. 
% Specifically, if $Q = [I_{n-1}\,|\,R]$ then
% $C = [-R^T\,|\,I_{m-n+1}]$ (see Figure \ref{f:1}, bottom right). 
% Hence $QC = 0$, i.e., fundamental cutsets and fundamental circuits are orthogonal vectors. 
Moreover, 
%what matters the most to us is that 
any cut (viewed as an $m$-dimensional vector)
in the graph can be represented as a linear
combination of fundamental cutset vectors with linear combination coefficients $-1$, $+1$, and $0$.
In essence, cuts are an $(n-1)$-dimensional vector space with fundamental cutsets (rows $Q_i$ of Q) as basis.
% For instance, if we let $Q_i$ denote the $i$-th row of matrix $Q$, then
% the depicted 
% red cut containing nodes $3$, $4$, and $5$ (which is not a cutset) can be generated 
% as $-Q_2-Q_3-Q_4+Q_6$, the 
% blue cutset containing nodes $4$, $5$, $6$, and $7$ is represented
% as $-Q_3 - Q_4 - Q_5 - Q_6$. The cutset containing nodes $4$ and $5$ is represented
% as $-Q_3 - Q_4$. 
It is important to observe that the vectors $Q_i$ involved in this representation
are precisely those corresponding to the branches of $T$ that are either moving inward (coefficient $-1$)
or outward (coefficient $+1$). Hence the fewer are the branches of $T$ cutting inward or outward, 
the sparser is this representation.
% In a similar fashion, any circuit vector can be represented as a linear combination of fundamental 
% circuit vectors
% with linear combination coefficients $-1$, $+1$, and $0$, i.e., circuits are an $(m-n+1)$-
% dimensional vector 
% space with fundamental circuits as basis, which is {\em orthogonal} to the vector space generated 
% by fundamental
% cuts. Since the dimension of the set of all subsets of $E$ is $m$ (when viewed as a vector space), 
% {\em any} subset of edges (viewed again as a vector) can be represented as the sum of a cut and a 
% circuit.
% This representation need not be unique, though. Matrices 
Matrix $Q$ 
% and $C$ have 
has also further properties, like
total unimodularity. This implies that any linear combination of their rows 
with coefficients in $\{-1,0,+1\}$
will result in a vector whose coefficients are again in $\{-1,0,+1\}$.

To summarize, given a spanning tree $T$ of $G$, a direction for $G$'s edges, and the associated matrix 
$Q$, any cutset\footnote
{
Though we are only interested in cutsets here, this statement holds more generally for any 
cut of the graph.
} 
$q$ in $G$ can be represented as an $m$-dimensional vector $\bq = Q^\top\,\bu$, where 
$\bu \in \{-1,0,+1\}^{n-1}$  has as many nonzero components as are the branches of $T$ belonging to $q$.
With this representation (induced by $T$) in hand, we are essentially aimed at learning in a 
sequential fashion $\bu$'s components. 
% Observe that $\bu$ is likely to be a sparse vector when the spanning tree $T$ happens to
% have little overlap with $q$. 

In order to tie this up with our similarity problem, we view the edges belonging to a given
cutset as the cut edges separating a (connected) cluster from the rest of the graph, and then
associate with any given $K$-labeling of the vertices of $G$ a sequence of $K$ weight vectors $\bu_k$,
$k = 1, \ldots, K$, each one corresponding to one label. Since a given label can spread over
multiple clusters (i.e., the vertices belonging to a given class label need not be a connected component
of $G$), we first need to collect connected
components belonging to the same cluster by summing the associated coefficient vectors.
%
% In fact, because we are framing this problem in the context of online
% matrix learning under linear separability assumptions, we need to further decompose the coefficient
% vector $\bu$ of the cut at hand into the sum of coefficient vectors corresponding to the cutsets 
% (or connected components) included in this cut. 
%@@
As an example, suppose in Figure \ref{f:4a} (a) 
we have 3 vertex labels corresponding to the three colors.
The blue cluster contains vertices $4,5,6$ and $7$, the green one vertex $1$, and the red one vertices 
$2$ and $3$. Now, whereas the blue and the green labels are connected, the red one is not. 
Hence we have $\bu_{\{4,5,6,7\}} = (0,0,-1,-1,-1,-1)^\top$, $\bu_{\{1\}} = (1,1,1,1,0,0)^\top$,
and $\bu_{\{2,3\}} = (-1,-1,0,0,1,1)^\top$ is the sum of the two cutset coefficient vectors
$\bu_{\{2\}} = (-1,0,0,0,1,0)^\top$, and $\bu_{\{3\}} = (0,-1,0,0,0,1)^\top$.
%
% the red cut $\{3,4,5\}$ in 
% (top left) is represented by $\bu_{\{3,4,5\}} = [0,-1,-1,-1,0,1]^\top$, which can be decomposed as
% $\bu_{\{3\}} + \bu_{\{4,5\}}$, with $\bu_{\{3\}} = [0,-1,0,0,0,1]^\top$
% and $\bu_{\{4,5\}} = [0,0,-1,-1,0,0]^\top$ are the coefficients corresponding to cutsets $\{3\}$, and $\{4,5\}$,
% respectively. 
%
% In general, any cut $q$ with $K$ connected components can be seen as partitioned into $K$
% cutsets $q_1, q_2, \ldots, q_K$ such that $\bu_q = \sum_{k = 1}^K \bu_{q_k}$. 
% The vectors $\bu_{q_k}$ are
% again sparse whenever the original vector $\bu_q$ is. Because of the above properties, 
%
In general, our goal will then be to learn a sparse (and rank-$K$) matrix 
$U = \sum_{k = 1}^K \bu_{k}\bu_{k}^\top$, where $\bu_k$ corresponds to the $k$-th (connected or 
disconnected) class label.

% Since in our online protocol the algorithm receives pairs of nodes $(i,j)$, the problem arises as to how 
% represent them. A natural representation (which, however, only works so far for the two-component case, i.e., 
% only for the problem of learning cutsets) 

Consistent with the above, we represent the pair of vertices $(i_t,j_t)$ as an indicator 
vector encoding the unique path in $T$ 
that connects the two vertices. This encoding takes edge orientation into account.
For instance, the pair of vertices $(6,7)$ in Figure \ref{f:4a} (c) is connected in $T$ by path 
$p_{(6\rightarrow 7)} = 6 \rightarrow 3 \rightarrow 1 \rightarrow 2  \rightarrow 7$.
According to the direction of traversed edges (edge $e_1$ is traversed according to its orientation, 
edge $e_2$ in the opposite direction, etc.), path $p_{(6\rightarrow 7)}$ is represented by vector
$\bp = (1,-1,0,0,1,-1,0,0,0,0,0,0)^\top = ((\bp'_t)^\top | 0^\top_{m-n+1})$, hence 
$Q\bp = \bp'_t = (1,-1,0,0,1,-1)^\top$.\footnote
{
Any other path connecting $6$ to $7$ in $G$ would yield the same representation. For instance,
going back to Figure \ref{f:4a} (a), consider path $p'_{(6\rightarrow 7)} = 6 \rightarrow 4 \rightarrow 7$, 
whose edges are not in $T$. This gives $\bp' = (0, 0, 0, 0, 0, 0, 0, 0, -1, 1, 0, 0)^\top$. Yet,
$Q\bp' = Q\bp = (1,-1,0,0,1,-1)^\top$.
This invariance holds in general: Given the pair $(i,j)$, the quantity $Q \bp$ is independent of $\bp$, 
if we let $\bp$ vary over all paths in $G$ departing from $i$ and arriving at $j$. 
This common value is the $(n-1)$-dimensional
vector containing the edges in the unique path in $T$ joining $i$ and $j$ 
(taking traversal directions into account). 
Said differently, once we are given $T$, the quantity $Q \bp$ only depends on $i$ and $j$, 
not on the path chosen to connect them.
This invariance easily follows from the fact that cuts are orthogonal to circuits, see,
e.g., \cite[Ch.6]{gy03}.
% and circuits are orthogonal: If $\bp$ connects $i$
% to $j$ and $\bp'$ connects $j$ to $i$ in such a way that $\bp$ and $\bp'$ together form a circuit 
% (no repeated edges), then we must have $Q(\bp+\bp') = 0$, i.e.,   $Q\bp = -Q\bp'$. We then flip 
% the sign in $-Q\bp'$ since $\bp'$ goes from $j$ back to $i$.
}
It is important to observe that computing $Q\bp$ does not require full knowledge of matrix $Q$,
since $Q\bp$ only depends on $T$ and the way its edges are traversed.
With the above handy, we are ready to prove Theorem \ref{t:unknown}.
\begin{figure}[h!]
\begin{picture}(0,290)(0,290)
\scalebox{0.7}{\includegraphics{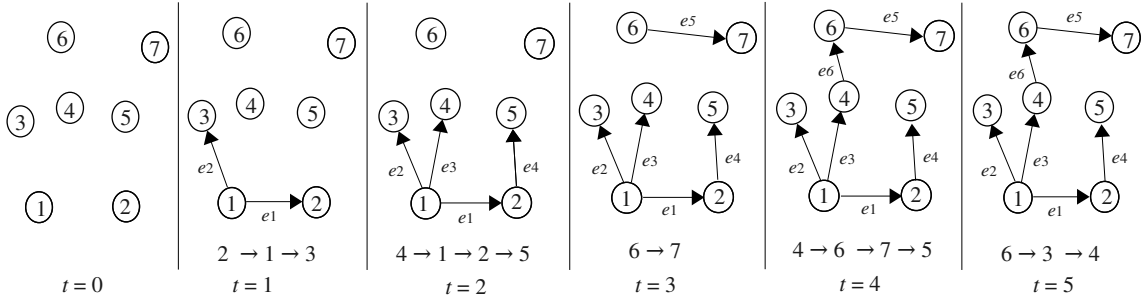}}
\end{picture}
\vspace{-2.1in}
\caption{
The way the $r$-norm Perceptron for similarity prediction of Section \ref{s:unknown} builds
instance vectors. 
At the beginning of the game ($t=0$) the algorithm is only
aware of the number of vertices ($n=7$ in this case). 
In round $t=1$ the pair $(2,3)$ is generated, along with the connecting path $2 \rightarrow 1 \rightarrow 3$.
Since none of the two revealed edges is creating cycles, the two edges $(1,2)$ and $(1,3)$ are
added to the forest in this order. Hence, $\bx_1 = (-1,1,0,0,0,0)^\top$. 
Round $t=2$: pair $(4,5)$ and path 
$4 \rightarrow 1 \rightarrow 2 \rightarrow 5$ are disclosed. 
Edges $(1,4)$ and $(2,5)$ are revealed to the
algorithm for the first time.
The associated vector is then $\bx_2 = (1,0,-1,1,0,0)^\top$. 
Round $t=3$: A new edge is revealed which is disconnected from the
previous subtree. We have $\bx_3 = (0,0,0,0,1,0)^\top$. 
Round $t=4$: The algorithm receives pair $(4,5)$ and corresponding
path $4 \rightarrow 6 \rightarrow 7 \rightarrow 5$. While edge $(4,6)$ is added to the forest,
causing the two subtrees to merge, neither edge $(6,7)$ nor edge $(7,5)$ is added. In particular,
$(7,5)$ is not added because of the presence of an alternative path in the current forest
(which is now a single tree) joining the two vertices. Hence the observed path 
$4 \rightarrow 6 \rightarrow 7 \rightarrow 5$ gets replaced by path
$4 \rightarrow 1 \rightarrow 2 \rightarrow 5$, and the corresponding instance vector is
$\bx_4 = (1,0,-1,1,0,0)^\top$.
Round $t=5$: Since we have obtained a spanning tree, 
from this point on, no other edges will be added. In this round 
we have $\bx_5 = (0,0,0,0,0,-1)^\top$, since the alternative (single edge) path 
$6 \rightarrow 4$ connecting $6$ to $4$ is already contained in the tree.
%\vspace{-0.2in}
}
\label{f:3}
\end{figure}

\ \\
\noindent{\bf Proof of Theorem \ref{t:unknown} }
%Let the disconnected edges decompose $G$ into $K$ connected components.
For the constructed spanning tree\footnote
{
If less than $n-1$ edges end up being revealed, the set of edges maintained by the algorithm
cannot form a spanning tree of $G$. Hence $T$ can be taken to be any spanning tree
of $G$ including all the revealed edges.
} 
$T$, let $\bu_k \in \{-1,0,1\}^{n-1}$ be the vector of coefficients
representing the $k$-th class label
%(viewed as a cutset separating this component from the rest of the graph) 
w.r.t. the fundamental cutset matrix $Q$ associated with $T$, and set 
$U = \sum_{k=1}^K \bu_k\bu_k^\top$. Also, let $\bx_t$ be the 
instance vector computed by the algorithm at time $t$.
Observe that, by the way $\bx_t$ is constructed (see Figure \ref{f:3} for an illustrative example) 
we have $\bx_t = Q\bp_t = \bp'_t$ for all $t$, being $\bp_t$ and $\bp'_t$ the path vectors
alluded at above. For any given class $k$, we have that $\bu_k^\top\bx_t =\bu_k^\top\bp'_t$. 
Recall that vector $\bu_k$
contains $+1$ in each component corresponding to an outward branch of $T$, $-1$ in each 
component corresponding
to an inward branch, and 0 otherwise. We distinguish four cases (see Figure \ref{f:4a} (c),
for reference):
\begin{enumerate}
\item $i_t$ and $j_t$ are both in the $k$-th class. In this case, the path in $T$ that connects
$i_t$ to $j_t$ must exit and enter the $k$-th class the same number of times (possibly zero). 
Since we only traverse
branches, we have in the dot product $\bu_k^\top\bp'_t$
an equal number of $+1$ terms (corresponding to departures from the $k$-th class)
and $-1$ (corresponding to arrivals). Hence $\bu_k^\top\bp'_t = 0$. Notice that this applies 
even when the $k$-th class is not connected.
\item $i_t$ is in the $k$-th class, but $j_t$ is not. 
In this case, the number of departures from the $k$-th class should exceed the number of arrivals
by exactly one. Hence we must have $\bu_k^\top\bp'_t = 1$.
\item $i_t$ is not in the $k$-th class, but $j_t$ is. By symmetry (swapping $i_t$ with $j_t$),
we have $\bu_k^\top\bp'_t = -1$.
\item Neither $i_t$ nor $j_t$ is in the $k$-th class. Again, we have an equal number of 
arrival/departures to/from the $k$-th class (possibly zero), hence $\bu_k^\top\bp'_t = 0$.
\end{enumerate}
We are now in a position to state our linear separability condition. We can write
\begin{eqnarray*}
\vct(U)^\top\vct(X_t) 
= \tr(U^\top X_t)
= \tr(\sum_{k=1}^K \bu_k\bu_k^\top \bx_t\bx_t^\top)
= \sum_{k=1}^K (\bu_k^\top \bx_t)^2
= \sum_{k=1}^K (\bu_k^\top \bp'_t)^2,
\end{eqnarray*}
which is 2 if $i_t$ and $j_t$ are in different classes (i.e., $i_t$ and $j_t$ are dissimilar), 
and 0, otherwise (i.e., $i_t$ and $j_t$ are similar). 
We have therefore obtained that the label $y_t$ associated with $(i_t,j_t)$ 
is delivered by the following linear-threshold function:
\begin{equation}\label{e:separ}
y_t =
\begin{cases}
1 &{\mbox{if $\vct(U)^\top X_t \geq 1$ }}\\
0 &{\mbox{otherwise}}\,.
\end{cases}
\end{equation}
Because we can interchangeably view $\vct(\cdot)$ as vectors or matrices, this opens up the possibility of running any linear-threshold learning algorithm (on either vectors or
matrices). For $r$-norm Perceptrons with the selected norm $r$ and decision 
threshold, we have a bound on the number $M$ of mistakes of the form \cite{gls97,g03}
\[
M = \cO\left( ||\vct(U)||^2_1\,||\vct(X_t)||^2_{\infty}\,\log n\right)~,
\]
where
\[
||\vct(U)||_1 = ||\vct(\sum_{k=1}^K \bu_k\bu_k^\top)||_1 =  ||\sum_{k=1}^K \vct(\bu_k\bu_k^\top)||_1
\leq \sum_{k=1}^K ||\vct(\bu_k\bu_k^\top)||_1 = \sum_{k=1}^K ||\bu_k||^2_1~,
\]
and 
\[
||\vct(X_t)||_{\infty} = ||\vct(\bx_t\bx_t^\top)||_{\infty} = ||\bx_t||^2_{\infty} = 1~.
\]
Moreover, by the way vectors $\bu_k$ are constructed, 
we have $||\bu_k||_1 = |\Phi_k^T|$. In turn, $|\Phi_k^T| \leq |\Phi_k^G|$ holds
independent of the connectedness of the $k$-th cluster. Putting
together and upper bounding concludes the proof.

\end{document}